%% file: main.tex
\documentclass{article}
\usepackage[dvipsnames]{xcolor}
\usepackage{tikz}
\usepackage{pgfplots}
\pgfplotsset{compat=1.15}

\usepackage{answers}
\usepackage{xcolor}
\usepackage{setspace}
\usepackage{graphicx}
\usepackage[shortlabels]{enumitem}
\usepackage{multicol}
\usepackage{silence}
\usepackage{mathrsfs}
\usepackage[margin=1in]{geometry} 
\usepackage{amsmath,amsthm,amssymb}
\usepackage{mathtools}
\usepackage[utf8]{inputenc}
\usepackage[english]{babel}
\usepackage{caption}
\usepackage{subcaption}
\usepackage[square,sort,comma,numbers]{natbib}
\usepackage{bbm}
\usepackage{breqn}
\usepackage{subfiles}
\usepackage{amsmath,amssymb}
\usepackage[english]{babel}
\usepackage[nottoc]{tocbibind}

\usepackage[toc,page]{appendix}
\newcommand{\stoptocwriting}{%
  \addtocontents{toc}{\protect\setcounter{tocdepth}{-5}}}
\newcommand{\resumetocwriting}{%
  \addtocontents{toc}{\protect\setcounter{tocdepth}{\arabic{tocdepth}}}}
\addto{\captionsenglish}{}

\usepackage{thmtools}
\usepackage{thm-restate}

\usepackage{hyperref}
\hypersetup{unicode}

\usepackage{cleveref}

\usetikzlibrary{patterns}

\let\tilde\widetilde

\newtheorem{theorem}{Theorem}[section]
\newtheorem{corollary}{Corollary}[theorem]
\newtheorem{lemma}[theorem]{Lemma}
\newtheorem{fact}[theorem]{Fact}
\newtheorem{definition}[theorem]{Definition}

\newtheorem{remark}[theorem]{Remark}
\newtheorem{claim}[theorem]{Claim}

\newtheorem{condition}{Condition}

\DeclareMathOperator*{\poly}{\textnormal{poly}}

\newcommand{\R}{\mathbb{R}}

\let\Pr\relax
\DeclareMathOperator*{\Pr}{\mathbb{P}}

\newcommand{\eps}{\varepsilon}
\newcommand{\norm}[1]{\|#1 \|}

\DeclareMathOperator*{\E}{\mathbb{E}}

\newcommand{\tr}{\textnormal{Tr}}

\newcommand{\inprod}[1]{\left\langle #1 \right\rangle}
\newcommand{\sign}{\textnormal{Sign}}

\allowdisplaybreaks

\input{libtheorems}

\title{Neural Networks with Sparse Activation Induced by Large Bias: Tighter Analysis with Bias-Generalized NTK}

\author{Hongru Yang \\ UT Austin \\\texttt{hy6385@utexas.edu} 
\and 
Ziyu Jiang \\ NEC Labs America \\ \texttt{jiangziyu@tamu.edu} 
\and
Ruizhe Zhang \\ Simons Institute, \\ UC Berkeley \\
\texttt{rzzhang@berkeley.edu}
\and
Yingbin Liang\\OSU \\\texttt{liang.889@osu.edu}
\and
Zhangyang Wang\\UT Austin  \\\texttt{atlaswang@utexas.edu}
}

\begin{document}

\maketitle

\begin{abstract}
We study training one-hidden-layer ReLU networks in the neural tangent kernel (NTK) regime, where the networks' biases are initialized to some constant rather than zero.
We prove that under such initialization, the neural network will have sparse activation throughout the entire training process, which enables fast training procedures
via some sophisticated computational methods. With such initialization, we show that the neural networks possess a different limiting kernel which we call \textit{bias-generalized NTK}, and we study various properties of the neural networks with this new kernel.
We first characterize the gradient descent dynamics. 
In particular, we show that the network in this case can achieve as fast convergence as the dense network, as opposed to the previous work suggesting that the sparse networks converge slower. 
In addition, our result improves the previous required width to ensure convergence.
Secondly, we study the networks' generalization: we show a width-sparsity dependence, which yields a sparsity-dependent Rademacher complexity and generalization bound. 
To our knowledge, this is the first sparsity-dependent generalization result via Rademacher complexity. 
Lastly, we study the smallest eigenvalue of this new kernel.
We identify a data-dependent region where we can derive a much sharper lower bound on the NTK's smallest eigenvalue than the worst-case bound previously known. This can lead to improvement in the generalization bound. 
\end{abstract}

\input{intro}
\input{prelim}
\input{convergence}
\input{generalization}

\input{exp}

\section*{Acknowledgement}
The work of Z. Wang was in part supported by an NSF Scale-MoDL grant (award number: 2133861) and the CAREER Award (award number: 2145346).
The work of R. Zhang was supported by DOE grant No. DE-SC0024124. 
The work of Y. Liang was supported in part by the U.S. National Science Foundation under the grants CCF-1900145, ECCS-2113860, and DMS-2134145.
\stoptocwriting

%
\newpage
\bibliographystyle{alpha}
\bibliography{ref}
\newpage
\renewcommand{\contentsname}{Contents}
 \tableofcontents
 \resumetocwriting
\newpage
\appendix

\input{appendix}
\input{probability_appendix}
\input{advantage_sparsity_appendix}

\end{document}

%% file: libtheorems.tex
\usepackage{etoolbox}

\makeatletter

\newcommand{\define}[4][ignore]{%
  \ifstrequal{#1}{ignore}{}{
  \@namedef{thmtitle@#2}{#1}}%
  \@namedef{thm@#2}{#4}%
  \@namedef{thmtypen@#2}{lemma}%
  \newtheorem{thmtype@#2}[theorem]{#3}%
  \newtheorem*{thmtypealt@#2}{#3~\ref{#2}}%
}

\newcommand{\state}[1]{%
  \@namedef{curthm}{#1}
  \@ifundefined{thmtitle@#1}{
  \begin{thmtype@#1}
    }{
  \begin{thmtype@#1}[\@nameuse{thmtitle@#1}]
  }
    \label{#1}
    \@nameuse{thm@#1}
  \end{thmtype@#1}
  \@ifundefined{thmdone@#1}{
  \@namedef{thmdone@#1}{stated}%
  }{}
}

\newcommand{\restate}[1]{%
  \@namedef{curthm}{#1}
  \@ifundefined{thmtitle@#1}{
    \begin{thmtypealt@#1}
    }{
  \begin{thmtypealt@#1}[\@nameuse{thmtitle@#1}]
  }
    \@nameuse{thm@#1}
  \end{thmtypealt@#1}
  \@ifundefined{thmdone@#1}{
  \@namedef{thmdone@#1}{stated}%
  }{}
}

\newcommand{\thmlabel}[1]{
  \@ifundefined{thmdone@\@nameuse{curthm}}{\label{#1}
    }{\tag*{\eqref{#1}}}
}
\makeatother

%% file: intro.tex
\section{Introduction}
The literature of sparse neural networks can be dated back to the early work of \cite{lecun1989optimal} where they showed that a fully-trained neural network can be pruned to preserve generalization. 
Recently, training sparse neural networks has been receiving increasing attention since the discovery of the lottery ticket hypothesis \citep{frankle2018lottery}. 
The lottery ticket hypothesis shows that there exists a sparse network inside a dense network at the initialization such that when the sparse network is trained, it can match the performance of the dense network. This discovery has spurred a lot of interest in the deep learning community as now sparse networks can not only bring computational benefits during inference time but also at training. 
However, their method of finding such sparse network requires multiple rounds of training and pruning, which is computationally expensive for any practical purposes. 
Nonetheless, this inspires further interest in the machine learning community to develop efficient methods to find sparse networks at the initialization such that the performance of the sparse network can match the dense network after training \citep{lee2018snip, wang2019picking, tanaka2020pruning, liu2020finding, chen2021only, he2017channel, liu2021unreasonable}. 

On the other hand, instead of trying to find some desirable sparse networks at the initialization, another line of research has been focusing on introducing sparsity at the random initialization and the sparsity is automatically maintained during training.
The key observation is that if the neural network activation is sparse during the entire training process, then only the weights of the activated neurons (i.e., ReLU will output some non-zero value) needs to be updated and one can utilize the sparsity to speedup per-step gradient descent training via techniques such as high-dimensional geometric data structures, sketching or even quantum algorithms \citep{song2021does, song2024training, hu2022training, gao2022sublinear, alman2024bypass}. 
In this line of theoretical studies, the sparsity is induced by the shifted ReLU which is the same as initializing the bias of the network's linear layer to some large constant and holding the bias fixed throughout the entire training. 
By the concentration of Gaussian, at the initialization, the total number of activated neurons will be \textit{sublinear} in the total number $m$ of neurons, if we initialize the bias by $C\sqrt{\log m}$ for some appropriate constant $C$. 
We call this {sparsity-inducing initialization}.
If the network is in the NTK regime, each neuron weight will exhibit small change after training, and thus the sparsity can be preserved throughout the entire training process. 
Therefore, at each step of gradient descent, only a sublinear number of the neuron weights need to be updated, which can significantly speedup the training process. 

The focus of this work is along the above line of theoretical studies of sparsely activated overparameterized neural networks and address the two main research limitations in the aforementioned studies: (1) 
prior work indicates that the sparse networks have {\bf slower convergence guarantee} than the dense network, despite that the per step gradient descent training can be made cheaper and (2)
the previous works only provided the convergence guarantee, while {\bf lacking the generalization analysis} which is of central interest in deep learning theory. 
Thus, our study fills the above important gaps, 
by first characterizing a new generalized limiting kernel of such type of neural networks and providing a comprehensive study with (a) finer analysis of the convergence; and (b) first generalization bound for such sparsely activated neural networks after training along with (c) sharp bound on the restricted smallest eigenvalue of the limiting NTK on some restricted region. 
We further elaborate our technical contributions are follows:
\vspace{-1mm}
\begin{enumerate}
    \item \textbf{Convergence.}
    In particular, \Cref{lemma: main_text_convergence} shows that the network with large bias initialization can achieve as fast convergence as the original network, as opposed to the previous work suggesting slower convergence.
    This is made possible by the fact that the sparse networks allow a much more relaxed condition on the learning rate, which was not discovered in the previous work. 
    The theorem further provides an improved required width to ensure that gradient descent can drive the training error towards zero at a linear rate. 
    This relies on our novel development of (1) a better characterization of the activation flipping probability via an analysis of the Gaussian anti-concentration based on the location of the strip and (2) a finer analysis of the initial training error.
    
    \item \textbf{Generalization.} \Cref{thm: main_text_generalization} studies the generalization of the network after gradient descent training where we characterize how the network width should depend on activation sparsity, which lead to a sparsity-dependent localized Rademacher complexity and generalization bound. 
    When the sparsity parameter is set to zero (i.e., the activation is not sparsified), our bound matches previous analysis up to logarithmic factors. 
    To our knowledge, this is the first sparsity-dependent generalization result via localized Rademacher complexity. 
    
    
    \item \textbf{Restricted Smallest Eigenvalue.}
    \Cref{thm: main_text_generalization} shows that the generalization bound heavily depends on the smallest eigenvalue $\lambda_{\textnormal{min}}$ of the limiting NTK. 
    However, the previously known worst-case lower bounds on $\lambda_{\textnormal{min}}$ under data separation
    have a $1/n^2$ explicit dependence in \citep{oymak2020toward, song2021does}, making the generalization bound vacuous. 
    Our \Cref{lemma: main_text_restricted_least_eig} establishes a much sharper lower bound that is sample-size-independent, on some data-dependent region.  
 This hence yields a worst-case generalization bound for \textit{bounded} loss of $O(1)$ as opposed to $O(n)$ in previous analysis, given that the label vector is in this region.
Since our new kernel is a generalized version of the previous kernel, our lower bound also provides improvements for the previous kernel.
    
    
\end{enumerate}
We include a comparison between our results and previous work in \Cref{table: bounds_comparison}.

\begin{table}\label[table]{table: bounds_comparison}
    \centering
    \begin{tabular}{|c|c|c|c|}
    \hline 
         & Width for convergence & Width for generalization & Large bias? \\
         \hline 
         \citep{du2018gradient}  & $\poly(\lambda_0^{-1}, n)$ & - & No \\
         \hline
         \citep{arora2019fine} & $\poly(\lambda_0^{-1}, n)$ & $\poly(\lambda_0^{-1}, n)$ & No \\
         \hline
         \citep{song2019quadratic} & $\widetilde{\Omega}\left( \lambda_0^{-4} n^4 \right)$ & $\widetilde{\Omega}(n^{16} \poly(1/\lambda(0))) $ & No \\
         \hline
         \citep{song2021does} & $\widetilde{\Omega}(\lambda_0^{-4} n^4 B^2 \exp(2B^2))$  & - & Yes \\
         \hline
         This work & $\widetilde{\Omega}\left( \lambda_0^{-4} n^4 \exp(B^2) \right)$ & $\widetilde{\Omega}\left( \lambda(B)^{-6} n^6 \exp(-B^2) \right) $  & Yes \\
         \hline
    \end{tabular}
    \caption{Comparison of results with previous work.}
    \label{tab:my_label}
\end{table}

\textbf{Practicality of NTK theory.} Although many works (such as \cite{chizat2019lazy}) pointed out that the NTK regime is a ``lazy training'' regime and cannot fully explain the success of deep learning in practice, it has become less well-known these days on the \textbf{utility} of NTK theory.
First of all, there are many works showing that for certain cases, replacing neural networks by NTK or other suitable kernels exhibits only limited performance drop \citep{shankar2020neural, novak2018bayesian, li2019enhanced, garriga2018deep, matthews2018gaussian, lee2018deep, lee2019wide, arora2020harnessing}. 
In particular, what's even more surprising is that \citep{ghorbani2021linearized} have shown that \textbf{NTK is minimax optimal} for learning dense polynomials. Therefore, although neural networks have shown impressive performances on many applications, there are still tasks on which NTK can perform on par with neural networks. Thus, our work can provide theoretical guidance for using large bias initialization to sparsify the activation of neural networks in the NTK regime.

\subsection{Related Works}
Besides the works mentioned in the introduction, another work related to ours is \citep{liao2022on} where they also considered training a one-hidden-layer neural network with sparse activation and studied its convergence. 
However, different from our work, their sparsity is induced by sampling a random mask at each step of gradient descent whereas our sparsity is induced by non-zero initialization of the bias terms. 
Also, their network has no bias term, and they only focus on studying the training convergence  but not generalization. 
We discuss additional related works here.

\textbf{Theory of neural tangent kernel.} 
A series of works have shown that if the neural network is wide enough (polynomial in depth, number of samples, etc), gradient descent can drive the training error towards zero in a fast rate either explicitly \citep{du2018gradient, du2019gradient, ji2019polylogarithmic} or implicitly \citep{allen2019convergence, zou2019improved, zou2020gradient} using the neural tangent kernel (NTK) \citep{jacot2018neural}. 
Further, under some conditions, the networks can generalize \citep{cao2019generalization}.
On the other hand, although NTK offers good convergence explanation, it contradicts the practice since (1) the neural networks need to be unrealistically wide and (2) the neuron weights merely change from the initialization. 
As \cite{chizat2019lazy} pointed out, the NTK regime is a ``lazy training'' regime which hardly explain the success of deep learning in practice.


\textbf{Sparse activation in neural networks.}
The sparse activation phenomena have been observed and utilized in practice. 
\citep{cao2019seernet} showed that it is possible to use a quantized network to predict the sparsity pattern of the activation from the original network, and this can be utilized to accelerate inference. 
Next, \citep{jaszczur2021sparse} forced the sparse activation of MLP in transformer to be static and used this to speed up the inference of transformers. 
Further, \citep{li2022lazy} systematically studied the sparse activation phenomena in transformers and showed that it occurs throughout a wide range of datasets and applications. 
They also showed that it can bring additional desired properties to manually introduce sparsity by selecting the top-$k$ largest values of the MLP activation.


%% file: prelim.tex
\section{Preliminaries}
\textbf{Notations.}
We use $\norm{ \cdot}_2$ to denote vector or matrix 2-norm and $\norm{\cdot}_F$ to denote the Frobenius norm of a matrix. 
When the subscript of $\norm{\cdot}$ is unspecified, it is default to be the 2-norm. For matrices $A \in \R^{m \times n_1}$ and $B \in \R^{m  \times n_2}$, we use $[A, B]$ to denote the row concatenation 
of $A,B$ and thus $[A, B]$ is a $m \times (n_1 + n_2)$ matrix. 
For matrix $X \in \R^{m \times n}$, the row-wise vectorization of $X$ is denoted by $\vec{X} = [x_1, x_2, \ldots, x_m]^\top$ where $x_i$ is the $i$-th row of $X$. 
For a given integer $n \in \mathbb{N}$, we use $[n]$ to denote the set $\{0, \ldots, n\}$, i.e., the set of integers from $0$ to $n$. 
For a set $S$, we use $\overline{S}$ to denote the complement of $S$. 
We use ${\cal N}(\mu,\sigma^2)$ to denote the Gaussian distribution with mean $\mu$ and standard deviation $\sigma$. 
In addition, we use $\widetilde{O}, \widetilde{\Theta}, \widetilde{\Omega}$ to suppress (poly-)logarithmic factors in $O, \Theta, \Omega$. 

\subsection{Problem Formulation}\label{sec: problem_formulation}
Let the training set to be $(X,y)$ where $X = (x_1, x_2, \ldots, x_n) \in \R^{d \times n}$ denotes the feature matrix consisting of $n$ $d$-dimensional vectors, and $y = (y_1, y_2, \ldots, y_n) \in \R^n$ consists of the corresponding $n$ response variables. 
We assume $\norm{x_i}_2 \leq 1$ and $y_i = O(1)$ for all $i \in [n]$. 
We use one-hidden-layer neural network and consider the regression problem with the square loss function: 
\begin{align*}
    f(x; W, b) &:= \frac{1}{\sqrt{m}} \sum_{r=1}^m a_r \sigma(\inprod{w_r, x} - b_r), \\
    L(W, b) &:= \frac{1}{2} \sum_{i=1}^n (f(x_i; W, b) - y_i)^2,
\end{align*}
where $W \in \R^{m \times d}$ with its $r$-th row being $w_r$, $b \in \R^m$ is a vector with $b_r$ being the bias of $r$-th neuron, $a_r$ is the second layer weight, and $\sigma(\cdot)$ denotes the ReLU activation function. 
We initialize the neural network by $W_{r,i} \sim \mathcal{N}(0, 1)$ and $a_r \sim \textnormal{Uniform}(\{\pm 1\})$ and $b_r = B$ for some value $B \geq 0$ of choice, for all $r \in [m],\ i \in [d]$. 
We train only the parameters $W$ and $b$ via gradient descent (i.e., with the linear layer $a_r,\ r \in [m]$ fixed), the updates are given by
\begin{align*}
    [w_r, b_r](t+1) &= [w_r, b_r](t) - \eta \frac{\partial L(W(t), b(t))}{\partial [w_r, b_r]}.
\end{align*}
By the chain rule, we have $\frac{\partial L}{\partial w_r} = \frac{\partial L}{\partial f} \frac{\partial f}{\partial w_r}$. 
The gradient of the loss with respect to the network is $\frac{\partial L}{\partial f} = \sum_{i=1}^n (f(x_i; W,b) - y_i)$ and the network gradients with respect to weights and bias are 
\begin{align*}
    \frac{\partial f(x; W, b)}{\partial w_r} &= \frac{1}{\sqrt{m}} a_r x \mathbb{I}(w_r^\top x \geq b_r), \\
    \frac{\partial f(x; W, b)}{\partial b_r} &= -\frac{1}{\sqrt{m}} a_r \mathbb{I}(w_r^\top x \geq b_r),
\end{align*}
where $\mathbb{I}(\cdot)$ is the indicator function. 
We use the shorthand $\mathbb{I}_{r,i} := \mathbb{I}(w_r^\top x_i \geq b_r)$ and define the empirical NTK matrix $H$ as 
\begin{align}\label{eq: NTK}
    H_{i,j}(W, b) &:= \inprod{\frac{\partial f(x_i; W, b)}{\partial [W, b]}, \frac{\partial f(x_j; W, b)}{\partial [W, b]}} = \frac{1}{m} \sum_{r=1}^m (\inprod{x_i, x_j} + 1) \mathbb{I}_{r,i} \mathbb{I}_{r,j}.
\end{align}
We define its infinite-width version $H^\infty(B)$, given by
\begin{align*}
    H^\infty_{ij}(B) &:= \E_{w} \left[ (\inprod{x_i, x_j} + 1) \mathbb{I}(w^\top x_i \geq B, w^\top x_j \geq B) \right].
\end{align*}
Notice that if we set $B=0$ and freeze the bias during training, we recover the usual NTK matrix studied in many previous literature such as \citep{du2018gradient}. Thus, we call our limiting matrix the \textbf{bias-generalized NTK}. 
Let $\lambda(B) := \lambda_{\textnormal{min}}(H^\infty(B))$. 
We define the matrix $Z(W,b) \in \R^{m(d+1) \times n}$ as
\begin{align*}
    &Z(W,b) := \frac{1}{\sqrt{m}} 
    \begin{bmatrix}
    \mathbb{I}_{1,1} a_1 \tilde{x}_1 & \ldots & \mathbb{I}_{1,n} a_1 \Tilde{x}_n \\
    \vdots & \ddots & \vdots \\
    \mathbb{I}_{m,1} a_m \tilde{x}_1 & \ldots & \mathbb{I}_{m,n} a_m \tilde{x}_n 
    \end{bmatrix},
\end{align*}
where $\Tilde{x}_i := [x_i^\top, -1]^\top$.
Note that $H(W,b) = Z(W,b)^\top Z(W,b)$. Hence, the gradient descent step can be written as
\begin{align*}
    \vec{[W,b](t+1)} = \vec{[W,b](t)} - \eta Z(t) (f(t) - y),
\end{align*}
where $[W,b](t) \in \R^{m \times (d+1)}$ denotes the row-wise concatenation of $W(t)$ and $b(t)$ at the $t$-th step of gradient descent, and $Z(t) := Z(W(t), b(t))$. 


%% file: convergence.tex
\section{Main Theory}
\subsection{Convergence and Sparsity}
We first present the convergence of gradient descent for the sparsely activated neural networks. 
Surprisingly, we show that the sparse network can achieve as fast convergence as the dense network compared to the previous work \citep{song2021does} which, on the other hand, shows the sparse networks converge slower than the dense networks. 
\begin{theorem}[Convergence]\label{lemma: main_text_convergence}
Let the learning rate {$\eta \leq O(\frac{\lambda(B) \exp(B^2)}{n^2})$}, and the bias initialization $B \in [0, \sqrt{0.5 \log m}]$. 
Assume $\lambda(B) = \lambda_0 \exp(-B^2/2)$ for some $\lambda_0>0$ independent of $B$. 
Then, if the network width satisfies $m \geq \widetilde{\Omega}\left( \lambda_0^{-4} n^4 \exp(B^2) \right)$,
with probability at least $1 - \delta - e^{-\Omega(n)}$ over the randomness in the initialization,
\begin{align*}
    \forall t: L(W(t), b(t)) \leq (1 - \eta \lambda(B) / 4)^t L(W(0), b(0)).
\end{align*}
\end{theorem}
The assumption on $\lambda(B)$ in \Cref{lemma: main_text_convergence} can be justified by \citep[Theorem F.1]{song2021does} which shows that under some conditions, the NTK's least eigenvalue $\lambda(B)$ is positive and has an $\exp(-B^2/2)$ dependence.
Given this, \Cref{lemma: main_text_convergence} in fact implies that the convergence rate is \textit{independent} of the sparsity parameter due to the extra $\exp(B^2)$ term in the learning rate. 
This means that the network with sparse activation can achieve as fast convergence as the original network. 
Our study further handles trainable bias (with constant initialization). 
This is done by a new result in \Cref{lemma: weight_bias_movement} that the change of bias is also diminishing with a $O(1/\sqrt{m})$ dependence on the network width $m$. 


\begin{remark}\label{rk:mbound}
\Cref{lemma: main_text_convergence} establishes a much sharper bound on the width of the neural network than previous work to guarantee the linear convergence. 
To elaborate, our bound only requires $m \geq \widetilde{\Omega}\left( \lambda_0^{-4} n^4 \exp(B^2) \right)$, as opposed to the bound $m \geq \widetilde{\Omega}(\lambda_0^{-4} n^4 B^2 \exp(2B^2))$ in \citep[Lemma D.9]{song2021does}. 
If we take $B = \sqrt{0.25\log m}$ (as allowed by the theorem), then our lower bound yields a polynomial improvement by a factor of $\widetilde{\Theta}(n/\lambda_0)^{8/3}$, which implies that the neural network width can be much smaller to achieve the same linear convergence.
\end{remark}

\subsubsection{Proof Outline of \texorpdfstring{\Cref{lemma: main_text_convergence}}{Convergence}}
\noindent
Like many previous NTK analysis, to prove convergence, we first characterize how many neurons we need so that the empirical NTK matrix (especially its minimum eigenvalue) is close to its infinite-width limit, in our case, the bias-generalized NTK (\Cref{lemma: diff_discrete_limit_ntk}). Then, we consider the case where all the possible neuron weights lying within some bounded region near their initialization values and we show that within this region the NTK's smallest eigenvalue is well above zero (\Cref{lemma: perturbed_ntk}). 
When we prove this result, we need to analyze how many neurons are activated and we derived a better bound on this neuron activation probability (\Cref{lemma: main_text_bound_flipping} below). 
Next, we use this smallest eigenvalue to show that the training loss can rapidly decrease toward zero when the neural network is within this region (\Cref{lemma: error_b4}). 
Along the way, we prove a better initial error bound (\Cref{claim: main_text_initial_error}) which leads to a better convergence guarantee.
Finally, since the training loss can decrease sufficiently fast, we can show that the changes of neuron weights during training is indeed small and the neural network is indeed within the region close to its initialization value via \Cref{lemma: weight_bias_movement}. 
We highlight our key results on novel analysis on activation flipping probability and a finer upper bound on initial error.

\subsubsection{Key Results in the Proof of \texorpdfstring{\Cref{lemma: main_text_convergence}}{Convergence}}
Like previous works, in order to prove convergence, we need to show that the NTK during training is close to its initialization. 
Inspecting the expression of NTK in \Cref{eq: NTK}, observe that the training will affect the NTK by changing the output of each indicator function. 
We say that the $r$-th neuron flips its activation with respect to input $x_i$ at the $k$-th step of gradient descent if $\mathbb{I}(w_r(k)^\top x_i - b_r(k) > 0) \neq \mathbb{I}(w_r(k-1)^\top x_i - b_r(k-1) > 0)$ for all $r \in [m]$.
The central idea is that for each neuron, as long as the weight and bias movement $R_w, R_b$ from its initialization is small, then the probability of activation flipping (with respect to random initialization) should not be large. 
We first present the bound on the probability that a neuron flips its activation. 
\begin{lemma}[Activation flipping probability]\label[lemma]{lemma: main_text_bound_flipping}
Let $B \geq 0$ and $R_w, R_b \leq \min\{1/B,1\}$. 
Let $\tilde{W} = (\widetilde{w}_1, \ldots, \tilde{w}_m)$ be vectors generated i.i.d. from $\mathcal{N}(0, I)$ and $\tilde{b} = (\tilde{b}_1, \ldots, \tilde{b}_m)= (B, \ldots, B)$, 
and weights $ W = (w_1, \ldots, w_m)$ and biases $ b = (b_1, \ldots, b_m) $ that satisfy for any $r \in [m]$, $\norm{\tilde{w}_r - w_r}_2 \leq R_w$ and $|\tilde{b}_r - b_r| \leq R_b$.
Define the event
\begin{align*}
    A_{i,r} &= \{\exists w_r, b_r: \norm{\tilde{w}_r - w_r}_2 \leq R_w,\ |b_r - \tilde{b}_r| \leq R_b, \mathbb{I}(x_i^\top \tilde{w}_r \geq \tilde{b}_r) \neq \mathbb{I}(x_i^\top w_r \geq b_r)\}.
\end{align*}
Then, for some constant $c$, 
\begin{align*}
    \Pr\left[ A_{i,r} \right] \leq c (R_w + R_b) \exp(-B^2/2).
\end{align*}
\end{lemma}
\begin{remark}
\citep[Claim C.11]{song2021does} presents a $O(\min\{R, \exp(-B^2/2)\})$ bound on $\Pr[A_{i,r}]$.
The reason that their bound involving the min operation is because $\Pr[A_{i,r}]$ can be bounded by the standard Gaussian tail bound and Gaussian anti-concentration bound separately and then, take the one that is smaller. 
On the other hand, our bound replaces the min operation by the product which creates a more convenient (and tighter) interpolation between the two bounds. 
Later, we will show that the maximum movement of neuron weights and biases, $R_w$ and $R_b$, both have a $O(1/\sqrt{m})$ dependence on the network width, and thus our bound offers a $\exp(-B^2/2)$ improvement where $\exp(-B^2/2)$ can be as small as $1/m^{1/4}$ when we take $B = \sqrt{0.5 \log m}$. 
\end{remark}
\textbf{Proof idea of \Cref{lemma: main_text_bound_flipping}.}
First notice that $\Pr[A_{i,r}] = \Pr_{x \sim \mathcal{N}(0,1)}[|x - B| \leq R_w + R_b]$. 
Thus, here we are trying to solve a fine-grained Gaussian anti-concentration problem with the strip centered at $B$. 
The problem with the standard Gaussian anti-concentration bound is that it only provides a worst case bound and, thus, is location-oblivious. 
Centered in our proof is a nice Gaussian anti-concentration bound based on the location of the strip, which we describe as follows:
Let's first assume $B > R_w + R_b$. 
A simple probability argument yields a bound of $2(R_w + R_b) \frac{1}{\sqrt{2\pi}} \exp(-(B - R_w - R_b)^2)$. 
Since later in the Appendix we can show that $R_w$ and $R_b$ have a $O(1/\sqrt{m})$ dependence (\Cref{lemma: weight_bias_movement} bounds the movement for gradient descent and \Cref{lemma: weight_bias_movement_gf} for gradient flow) and we only take $B = O(\sqrt{\log m})$, by making $m$ sufficiently large, we can safely assume that $R_w$ and $R_b$ is sufficiently small. 
Thus, the probability can be bounded by $O((R_w + R_b) \exp(-B^2/2))$. 
However, when $B < R_w + R_b$ the above bound no longer holds.
But a closer look tells us that in this case $B$ is close to zero, and thus $(R_w + R_b) \frac{1}{\sqrt{2\pi}} \exp(-B^2/2) \approx \frac{R_w + R_b}{\sqrt{2\pi}}$ which yields roughly the same bound as the standard Gaussian anti-concentration. 

Next, our analysis develops a finer initial error bound. 
\begin{lemma}[Initial error upper bound]\label[lemma]{claim: main_text_initial_error}
Let $B > 0$ be the initialization value of the biases and all the weights be initialized from standard Gaussian. 
Let $\delta \in (0,1)$ be the failure probability. Then, with probability at least $1 - \delta$ over the randomness in the initialization, we have
\begin{align*}
    L(0) &= {O}\left(n + n \left( \exp(-\frac{B^2}{2}) + \frac{1}{m} \right) \log^{3} (\frac{2mn}{\delta}) \right).
\end{align*}
\end{lemma}
\citep[Claim D.1]{song2021does} gives a rough estimate of the initial error with $O(n(1+B^2) \log^2(n / \delta) \log(m/\delta))$ bound. 
When we set $B = C\sqrt{ \log m}$ for some constant $C$, our bound improves the previous result by a polylogarithmic factor.
The previous bound is not tight in the following two senses: (1) the bias will only decrease the magnitude of the neuron activation instead of increasing and (2) when the bias is initialized as $B$, only roughly $O(\exp(-B^2/2)) \cdot m$ neurons will activate.  
Thus, we can improve the $B^2$ dependence to $\exp(-B^2/2)$. 

By combining the above two improved results, we can prove our convergence result with improved lower bound of $m$ as in \Cref{rk:mbound}. 
To relax the condition on the learning rate for the sparse network, a finer analysis of the error terms is conducted in \Cref{lemma: error_b4} by leveraging the fact that the network has sparse activation. 
This later translates into a wider range of learning rate choice in the convergence analysis. 
We provide the complete proof in \Cref{app: convergence}. 

Lastly, since we can show that the total movement of each neuron's bias has a $O(1/\sqrt{m})$ dependence (shown in \Cref{lemma: weight_bias_movement}), combining with the number of activated neurons at the initialization, we can bound the number of activated neurons. 
\begin{lemma}[Number of Activated Neurons per Iteration]
Assume the parameter settings in \Cref{lemma: main_text_convergence}. With probability at least $1 - e^{-\Omega(n)}$ over the random initialization, 
\begin{align*}
    |\mathcal{S}_{\textnormal{on}}(i,t)| = O(m \cdot \exp(-B^2/2)) 
\end{align*}
for all $0 \leq t \leq T$ and $i \in [n]$, where $\mathcal{S}_{\textnormal{on}}(i,t) = \{r \in [m]:\ w_r(t)^\top x_i \geq b_r(t)\}$. 
\end{lemma}
This lemma proves that the activation of the neural network remains sparse throughout the entire training process. Utilizing the computational techniques in the introduction, it can speed up the per step training of the neural network. 

%% file: generalization.tex
\subsection{Generalization Bound}
\subsubsection{Results}
In this section, we present our sparsity-dependent generalization result. 
For technical reasons stated in \Cref{sec: proof_generalization}, we use symmetric initialization defined below.
Further, {we adopt the setting in \citep{arora2019fine}} and use a non-degenerate data distribution to make sure the infinite-width NTK is positive definite. 
\begin{definition}[Symmetric Initialization]\label[definition]{def: symmetric_init}
For a one-hidden layer neural network with $2m$ neurons, the network is initialized as the following:
\begin{enumerate}
    \item For $r \in [m]$, independently initialize $w_r \sim \mathcal{N}(0, I)$ and $a_r \sim \textnormal{Uniform}(\{-1, 1\})$. 
    \item For $r \in \{m + 1, \ldots, 2m\}$, let $w_r = w_{r-m}$ and $a_r = -a_{r - m}$. 
\end{enumerate}
\end{definition}
\begin{definition}[$(\lambda_0, \delta, n)$-non-degenerate distribution, \citep{arora2019fine}]\label[definition]{def: non-degenerate_dist}
A distribution $\mathcal{D}$ over $\R^d \times \R$ is $(\lambda_0, \delta, n)$-non-degenerate, if for $n$ i.i.d. samples $\{(x_i, y_i)\}_{i=1}^n$ from $\mathcal{D}$, with probability $1 - \delta$ we have $\lambda_{\textnormal{min}}(H^\infty(B)) \geq \lambda_0 > 0$. 
\end{definition}
\begin{theorem}\label{thm: main_text_generalization}
Fix a failure probability $\delta \in (0,1)$ and an accuracy parameter $\eps \in (0, 1)$. 
Suppose the training data $S = \{(x_i,y_i)\}_{i=1}^n$ are i.i.d. samples from a $(\lambda, \delta, n)$-non-degenerate distribution $\mathcal{D}$ defined in \Cref{def: non-degenerate_dist}. 
Assume the one-hidden layer neural network is initialized by symmetric initialization in \Cref{def: symmetric_init}.
Further, assume the parameter settings in \Cref{lemma: main_text_convergence} except we let $m \geq \widetilde{\Omega}\left( \lambda(B)^{-6} n^6 \exp(-B^2) \right)$.
Consider any loss function $\ell: \R \times \R \rightarrow [0,1]$ that is $1$-Lipschitz in its first argument. 
Then with probability at least $1 - 2\delta - e^{-\Omega(n)}$ over the randomness in symmetric initialization of $W(0) \in \R^{m \times d}$ and $a \in \R^m$ and the training samples, the two layer neural network $f(W(t), b(t),a)$ trained by gradient descent for $t \geq \Omega(\frac{1}{\eta \lambda(B)} \log \frac{n \log (1/\delta)}{\eps})$ iterations has empirical Rademacher complexity (see its formal definition in \Cref{def: rademacher_complexity} in Appendix) bounded as
\begin{align*}
    & \mathcal{R}_S(\mathcal{F}) \leq \sqrt{\frac{y^\top (H^\infty(B))^{-1} y \cdot 8 e^{-B^2/2}}{n}} + \tilde{O}\left(\frac{e^{-B^2/4}}{n^{1/2}} \right)
\end{align*}
and the population loss $L_{\mathcal{D}}(f) = \E_{(x,y) \sim \mathcal{D}}[\ell(f(x), y)]$ can be upper bounded as
\begin{align}\label{eq: generlization}
    & L_{\mathcal{D}}(f(W(t), b(t), a))  \leq \sqrt{\frac{y^\top (H^\infty(B))^{-1} y \cdot 32 e^{-B^2/2}}{n}} + \tilde{O}\left(\frac{1}{n^{1/2}} \right). 
\end{align}
\end{theorem}

To show good generalization, we need a larger width: the second term in the Rademacher complexity bound is diminishing with $m$ and to make this term $O(1/\sqrt{n})$, the width needs to have $(n/\lambda(B))^6$ dependence as opposed to $(n/\lambda(B))^4$ for convergence. 
Now, at the first glance of our generalization result, it seems we can make the Rademacher complexity arbitrarily small by increasing $B$.
Recall from the discussion of \Cref{lemma: main_text_convergence} that the smallest eigenvalue of $H^\infty(B)$ also has an $\exp(-B^2/2)$ dependence. 
Thus, in the worst case, the $\exp(-B^2/2)$ factor gets canceled and sparsity will not hurt the network's generalization.

Before we present the proof, we make a corollary of \Cref{thm: main_text_generalization} for the zero-initialized bias case. 
\begin{corollary}\label[corollary]{cor:zeroB}
Take the same setting as in \Cref{thm: main_text_generalization} except now the biases are initialized as zero, i.e., $B = 0$. 
Then, if we let $m \geq \widetilde{\Omega}(\lambda(0)^{-6} n^6)$, the empirical Rademacher complexity and population loss are both bounded by
\begin{align*}
    & \mathcal{R}_S(\mathcal{F}),\ L_{\mathcal{D}}(f(W(t), b(t), a)) \leq \sqrt{\frac{y^\top (H^\infty(0))^{-1} y \cdot 32}{n}} + \tilde{O}\left(\frac{1}{n^{1/2}} \right).
\end{align*}
\end{corollary}
\Cref{cor:zeroB} requires the network width $m \geq \widetilde{\Omega}((n/\lambda(0))^6)$ which significantly improves upon the previous result in \citep[Theorem G.7]{song2019quadratic} $m \geq \widetilde{\Omega}(n^{16} \poly(1/\lambda(0)))$ (including the dependence on the rescaling factor $\kappa$) which is a much wider network.

{\bf Generalization Bound via Least Eigenvalue.} Note that in \Cref{thm: main_text_generalization}, the worst case of the first term in the generalization bound in \Cref{eq: generlization} is given by $\widetilde{O}(\sqrt{1/\lambda(B)})$.
Hence, the least eigenvalue $\lambda(B)$ of the NTK matrix can significantly affect the generalization bound. 
Previous works \citep{oymak2020toward, song2021does} established lower bounds on $\lambda(B)$ with an explicit $1/n^2$ dependence on $n$ under the $\delta$ data separation assumption (see \Cref{lemma: main_text_restricted_least_eig}), which clearly makes a vacuous generalization bound of $\widetilde{O}({n})$. 
This thus motivates us to provide a tighter bound (desirably independent on $n$) on the least eigenvalue of the infinite-width NTK in order to make the generalization bound in \Cref{thm: main_text_generalization} valid and useful. 
It turns out that there are major difficulties in proving a better lower bound in the general case.
However, we are only able to present a better lower bound when we restrict the domain to some (data-dependent) regions by utilizing trainable bias. 


\subsubsection{Key Ideas in the Proof of \texorpdfstring{\Cref{thm: main_text_generalization}}{Generalization}}\label{sec: proof_generalization}
Since each neuron weight and bias move little from their initialization, a natural approach is to bound the generalization via localized Rademacher complexity. 
After that, we can apply appropriate concentration bounds to derive generalization. 
The main effort of our proof is devoted to bounding the weight movement to bound the localized Rademacher complexity. 
If we directly take the setting in \Cref{lemma: main_text_convergence} and compute the network's localized Rademacher complexity, we will encounter a non-diminishing (with the number of samples $n$) term which can be as large as $O(\sqrt{n})$ since the network outputs non-zero values at the initialization. 
\cite{arora2019fine} and \cite{song2019quadratic} resolved this issue by initializing the neural network weights instead by $\mathcal{N}(0, \kappa^2I)$ to force the neural network output something close to zero at the initialization.
The magnitude of $\kappa$ is chosen to balance different terms in the Rademacher complexity bound in the end. 
Similar approach can also be adapted to our case by initializing the weights by $\mathcal{N}(0, \kappa^2 I)$ and the biases by $\kappa B$. 
However, the drawback of such an approach is that the effect of $\kappa$ to all the previously established results for convergence need to be carefully tracked or derived. 
In particular, in order to guarantee convergence, the neural network's width needs to have a polynomial dependence on $1/\kappa$ where $1/\kappa$ has a polynomial dependence on $n$ and $1/\lambda$, which means their network width needs to be larger to compensate for the initialization scaling. 
We resolve this issue by symmetric initialization \Cref{def: symmetric_init} which yields no effect (up to constant factors) on previously established convergence results, see \citep{munteanu2022bounding}. 
Symmetric initialization allows us to organically combine the results derived for convergence to be reused for generalization, which leads to a more succinct analysis. 
Further, we replace the $\ell_1$-$\ell_2$ norm upper bound by finer inequalities in various places in the original analysis.
All these improvements lead to the following upper bound of the weight matrix change in Frobenius norm. 
Further, combining our sparsity-inducing initialization, we present our sparsity-dependent Frobenius norm bound on the weight matrix change. 
\begin{lemma}\label[lemma]{lemma: main_text_W_fro_diff}
Assume the one-hidden layer neural network is initialized by symmetric initialization in \Cref{def: symmetric_init}. 
Further, assume the parameter settings in \Cref{lemma: main_text_convergence}. 
Then with probability at least $1 - \delta - e^{-\Omega(n)}$ over the random initialization, we have for all $t \geq 0$, 
\begin{align*}
    \norm{[W,b](t) - [W,b](0)}_F &\leq \sqrt{y^\top (H^\infty)^{-1} y} + O\left( \frac{n }{\lambda} \left( \frac{\exp(-B^2/2) \log(n/\delta)}{m} \right)^{\frac{1}{4}} \right) \\
    &\quad + O\left( \frac{n\sqrt{R \exp(-B^2/2)}}{\lambda} \right) \\
    &+ \frac{n}{\lambda^2} \cdot O\left( \exp(-B^2/4) \sqrt{\frac{\log(n^2/\delta)}{m}} + R\exp(-B^2/2) \right) 
\end{align*}
where $R = R_w + R_b$ denote the maximum magnitude of neuron weight and bias change.
\end{lemma}
By \Cref{lemma: weight_bias_movement} and \Cref{claim: initial_error} in the Appendix, we have $R = \widetilde{O}(\frac{n}{\lambda\sqrt{m}})$. 
Plugging in and setting $B = 0$, we get $ \norm{[W, b](t) - [W, b](0)}_F \leq \sqrt{y^\top (H^\infty)^{-1} y} + \widetilde{O}(\frac{n}{\lambda m^{1/4}} + \frac{n^{3/2}}{\lambda^{3/2} m^{1/4}} + \frac{n}{\lambda^2 \sqrt{m}} + \frac{n^2}{\lambda^3 \sqrt{m}})$. 
On the other hand, taking $\kappa = 1$, \citep[Lemma G.6]{song2019quadratic} yields a bound of $\norm{W(t) - W(0)}_F \leq \sqrt{y^\top (H^\infty)^{-1} y} + \widetilde{O}(\frac{n}{\lambda} + \frac{n^{7/2} \poly(1/\lambda)}{m^{1/4}})$. 
Notice that the $\widetilde{O}(\frac{n}{\lambda})$ term has no dependence on $1/m$ and is removed by symmetric initialization in our analysis.
We further improve the upper bound's dependence on $n$ by a factor of $n^2$. 

The full proof of \Cref{thm: main_text_generalization} is deferred in \Cref{app: generalization}.

\subsection{Restricted Least Eigenvalue of the Bias-Generalized NTK}
\subsubsection{Results}
\begin{definition}[Data-dependent Region]\label{def: main_text_data_dependent_region}
Let $p_{ij} = \Pr_{w \sim \mathcal{N}(0, I)}[w^\top x_i \geq B,\ w^\top x_j \geq B]$ for $i \neq j$.
Define the (data-dependent) region $\mathcal{R} = \{a \in \R^n:\ \sum_{i \neq j} a_i a_j p_{ij} \geq \min_{i' \neq j'} p_{i'j'} \sum_{i \neq j} a_i a_j\} $.
\end{definition}
Notice that $\mathcal{R}$ is non-empty for any input data-set since $\R^n_+ \subset \mathcal{R}$ where $\R^n_+$ denotes the set of vectors with non-negative entries, and $\mathcal{R} = \R^n$ if $p_{ij} = p_{i'j'}$ for all $i \neq i', j \neq j'$. 
\begin{theorem}[Restricted Least Eigenvalue]\label{lemma: main_text_restricted_least_eig}
Let $X = (x_1, \ldots, x_n)$ be points in $\R^d$ with $\norm{x_i}_2 = 1$ for all $i \in [n]$ and $w \sim \mathcal{N}(0, I_d)$. Suppose that there exists $\delta \in [0,\sqrt{2}]$ such that
\begin{align*}
    \min_{i \neq j \in [n]} (\norm{x_i - x_j}_2, \norm{x_i + x_j}_2) \geq \delta.
\end{align*}
Let $B \geq 0$. 
Consider the minimal eigenvalue of  $H^\infty$ over the data-dependent region $\mathcal R$ defined above, i.e., let $\lambda:= \min_{\norm{a}_2=1,\ a \in \mathcal{R}} a^\top H^\infty a$. Then, $\lambda \geq \max(0, \lambda')$ where
\begin{align}\label{eq: res_least_eig}
    \lambda' &\geq \max\left( \frac{1}{2} - \frac{B}{\sqrt{2\pi}},\ \left(\frac{1}{B} - \frac{1}{B^3} \right)\frac{e^{-B^2/2}}{\sqrt{2\pi}} \right) \nonumber\\
    &- e^{-B^2/(2-\delta^2/2)} \frac{\pi - \arctan \left( \frac{\delta\sqrt{1 - \delta^2/4}}{1 - \delta^2/2} \right) }{2\pi} .
\end{align}
\end{theorem}

To demonstrate the usefulness of our result, if we take the bias initialization $B = 0$ in \Cref{eq: res_least_eig}, this bound yields $1/(2\pi)\cdot {\arctan( (\delta\sqrt{1 - \delta^2/4})/(1 - \delta^2/2))} \approx \delta/(2\pi)$, when $\delta$ is close to $0$ whereas \citep{song2021does} yields a bound of $\delta/n^2$. 
On the other hand, if the data points are orthogonal, i.e., $\delta = \sqrt{2}$, we get a $ \max\left( \frac{1}{2} - \frac{B}{\sqrt{2\pi}},\ \left(\frac{1}{B} - \frac{1}{B^3} \right)\frac{e^{-B^2/2}}{\sqrt{2\pi}} \right)$ lower bound, whereas \citep{song2021does} yields a bound of $\exp(-B^2/2) \sqrt{2} /n^2$. 
Connecting to our convergence result in \Cref{lemma: main_text_convergence}, if $f(t) - y \in \mathcal{R}$, then the error can be reduced at a much faster rate than the (pessimistic) rate with $1/n^2$ dependence in the previous studies as long as the error vector lies in the region. 


\begin{remark}
The lower bound on the restricted smallest eigenvalue $\lambda$ in \Cref{lemma: main_text_restricted_least_eig} is {\bf independent of $n$},  
which makes that the worst case generalization bound in \Cref{thm: main_text_generalization} be $O(1)$ under constant data separation margin (note that this is optimal since the loss is \textit{bounded}). 
Such a lower bound is much sharper than the previous results with a $1/n^2$ explicit dependence which yields vacuous generalization bound of $O(n)$. 
This improvement relies on the condition that the label vector should lie in the region $\mathcal R$, which can be achieved by a simple label-shifting strategy:
Since $\R_+^n \subset \mathcal{R}$, the condition can be easily achieved by training the neural network on the shifted labels $y + C$ (with appropriate broadcast) where $C$ is a constant such that $ \min_i y_i + C \geq 0$. 
\end{remark} 
Careful readers may notice that in the proof of \Cref{lemma: main_text_restricted_least_eig} in \Cref{app: res_least_eig}, the restricted least eigenvalue on $\R_+^n$ is always positive even if the data separation is zero, which would imply that the network can always exhibit good generalization. 
However, we need to point out that the generalization bound in \Cref{thm: main_text_generalization} is meaningful only when the training is successful: when the data separation is zero, the limiting NTK is no longer positive definite and the training loss cannot be minimized toward zero.

\subsubsection{Key Ideas in the Proof of \texorpdfstring{\Cref{lemma: main_text_restricted_least_eig}}{Restricted Eigenvalue}}\label{sec: main_text_ideas_restricted_least_eig}
In this section, we analyze the smallest eigenvalue of the limiting NTK $H^\infty$ with $\delta$ data separation.
We first note that $H^\infty \succeq \E_{w \sim \mathcal{N}(0, I)} \left[ \mathbb{I}(Xw \geq B) \mathbb{I}(Xw \geq B)^\top \right]$ and for a fixed vector $a$, we are interested in the lower bound of $\E_{w \sim \mathcal{N}(0, I)}[|a^\top \mathbb{I}(Xw \geq B)|^2]$. 
In previous works, \cite{oymak2020toward} showed a lower bound $\Omega(\delta/ n^2)$  for zero-initialized bias, and later \cite{song2021does} generalized this result to a lower bound $\Omega(e^{-B^2/2} \delta/n^2)$ for non-zero initialized bias. 
Both lower bounds have a dependence of $1/n^2$. 
Their approach is by using an intricate Markov's inequality argument and then proving an lower bound of $\Pr[|a^\top \mathbb{I}(Xw \geq B)| \geq c \norm{a}_\infty]$. 
The lower bound is proved by only considering the contribution from the largest coordinate of $a$ and treating all other values as noise. 
It is non-surprising that the lower bound has a factor of $1/n$ since $a$ can have identical entries. 
On the other hand, the diagonal entries can give a $\exp(-B^2/2)$ upper bound and thus there is a $1/n^2$ gap between the two. 
Now, we give some evidence suggesting the $1/n^2$ dependence may not be tight in some cases. 
Consider the following scenario: Assume $n \ll d$ and the data set is orthonormal. 
For any unit-norm vector $a$, we have
\begin{align*}
    &a^\top \E_{w \sim \mathcal{N}(0, I)} \left[ \mathbb{I}(Xw \geq B) \mathbb{I}(Xw \geq B)^\top \right] a \\
    & \textstyle =  \sum_{i, j \in [n]} a_i a_j \Pr[w^\top x_i \geq B,\ w^\top x_j \geq B] \\
    &= p_0 \norm{a}_2^2 + p_1 \sum_{i \neq j} a_i a_j \\
    &\textstyle = p_0 - p_1 + p_1 \left( \sum_i a_i \right)^2 > p_0 - p_1
\end{align*}
where $p_0,p_1\in [0,1]$ are defined such that due to the spherical symmetry of the standard Gaussian we are able to let $p_0 = \Pr[w^\top x_i \geq B],\ \forall i \in [n]$ and $p_1 = \Pr[w^\top x_i \geq B, w^\top x_j \geq B],\ \forall i, j \in [n],\ i \neq j$. 
Notice that $p_0 > p_1$. 
Since this is true for all $a \in \R^n$, we get a lower bound of $p_0 - p_1$ with no explicit dependence on $n$ and this holds for all $n \leq d$. 
When $d$ is large and $n = d/2$, this bound is better than previous bound by a factor of $\Theta(1/d^2)$.
We hope to apply the above analysis to general datasets. 
However, it turns out that the product terms (with $i \neq j$) above creates major difficulties in the general case. 
Due to such technical difficulties, we prove a better lower bound by utilizing the data-dependent region $\mathcal{R}$ defined in \Cref{def: main_text_data_dependent_region}. 
Let $p_{\textnormal{min}} = \min_{i \neq j} p_{ij}$.
Now, for $a \in \mathcal{R}$, we have
\begin{align*}
    & \E_{w \sim \mathcal{N}(0, I)} \left[ (a^\top \mathbb{I}(Xw \geq B))^2 \right] \\
    &\geq (p_0 - p_{\textnormal{min}}) \norm{a}_2^2 + p_{\textnormal{min}} \norm{a}_2^2 + p_{\textnormal{min}} \sum_{i \neq j} a_i a_j \\
    &\geq (p_0 - \min_{i \neq j} p_{ij}) \norm{a}_2^2.
\end{align*}
Thus, to lower bound the smallest eigenvalue on this region, we need to get an upper bound on $\min_{i \neq j} p_{ij}$. 
To do this, let's first consider a fixed pair of training data $x_i$ and $x_j$ and their associated probability $p_{ij}$ (see \Cref{def: main_text_data_dependent_region}). 
To compute $p_{ij}$, we can decompose $x_j$ into two components: one is along the direction of $x_i$ and the other is orthogonal to $x_i$.
Now we can project the Gaussian vector onto these two directions and since the two directions are orthogonal, they are independent. 
This allows $p_{ij}$ to be computed via geometry arguments. 
It turns out that this probability is maximized when the data separation is the smallest. 
We defer the details of the proof of \Cref{lemma: main_text_restricted_least_eig} to \Cref{app: res_least_eig}. 


%% file: exp.tex
\section{Experiments}
In this section, we verify our result that the activation of neural networks remains sparse during training when the bias parameters are initialized as non-zero. 

\textbf{Settings.}
We train a 6-layer multi-layer perceptron (MLP) of width 1024 with trainable bias terms on MNIST image classification \citep{lecun2010mnist}. 
The biases of the fully-connected layers are initialized as $0, -0.5$ and $-1$. 
For the weights in the linear layer, we use Kaiming Initialization \citep{he2015delving} which is sampled from an appropriately scaled Gaussian distribution.
The traditional MLP architecture only has linear layers with ReLU activation. 
However, we found out that using the sparsity-inducing initialization, the magnitude of the activation will decrease geometrically layer-by-layer, which leads to vanishing gradients and that the network cannot be trained. 
Thus, we made a slight modification to the MLP architecture to include an extra Batch Normalization after ReLU to normalize the activation. 
Our MLP implementation is based on \citep{zhu2021geometric}. 
We train the neural network by stochastic gradient descent with a small learning rate 5e-3 to make sure the training is in the NTK regime. 
The sparsity is measured as the total number of activated neurons (i.e., ReLU outputs some positive values) divided by total number of neurons, averaged over every SGD batch. 
We plot how the sparsity patterns changes for different layers during training. 


\begin{figure}[ht]
  \centering
\begin{subfigure}{.33\columnwidth}
  \centering
  \includegraphics[width=.9\textwidth]{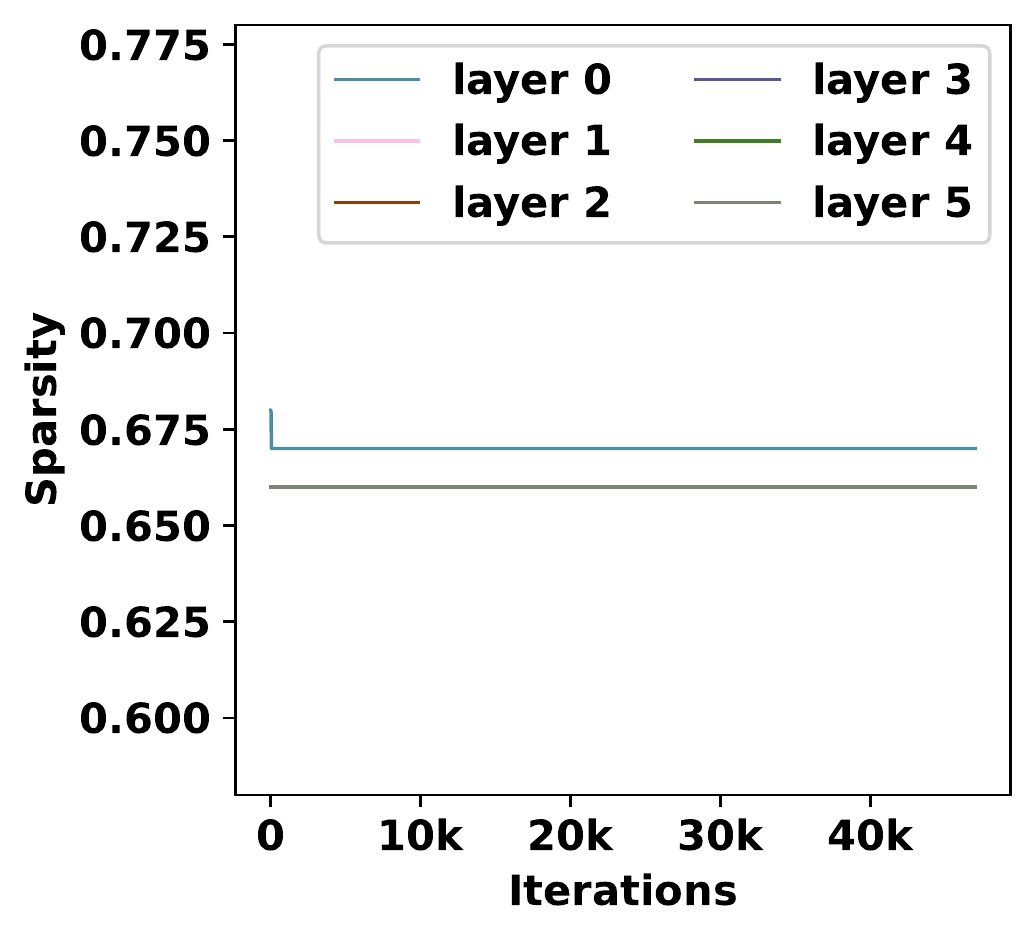}
  \caption{Init Bias as 0}
  \label{fig:bias0}
\end{subfigure}%
\begin{subfigure}{.33\textwidth}
  \centering
  \includegraphics[width=.9\textwidth]{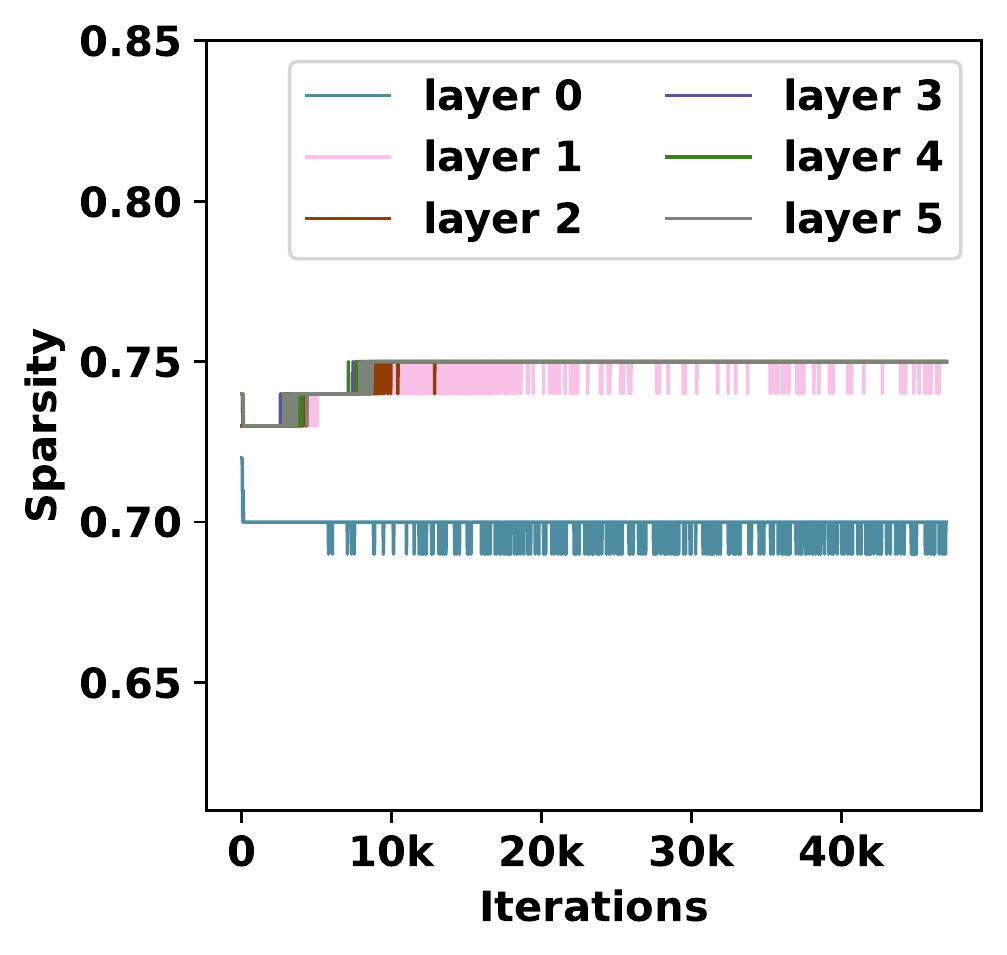}
  \caption{Init Bias as -0.5}
  \label{fig:bias-0.5}
\end{subfigure}%
\begin{subfigure}{.33\textwidth}
  \centering
  \includegraphics[width=.9\textwidth]{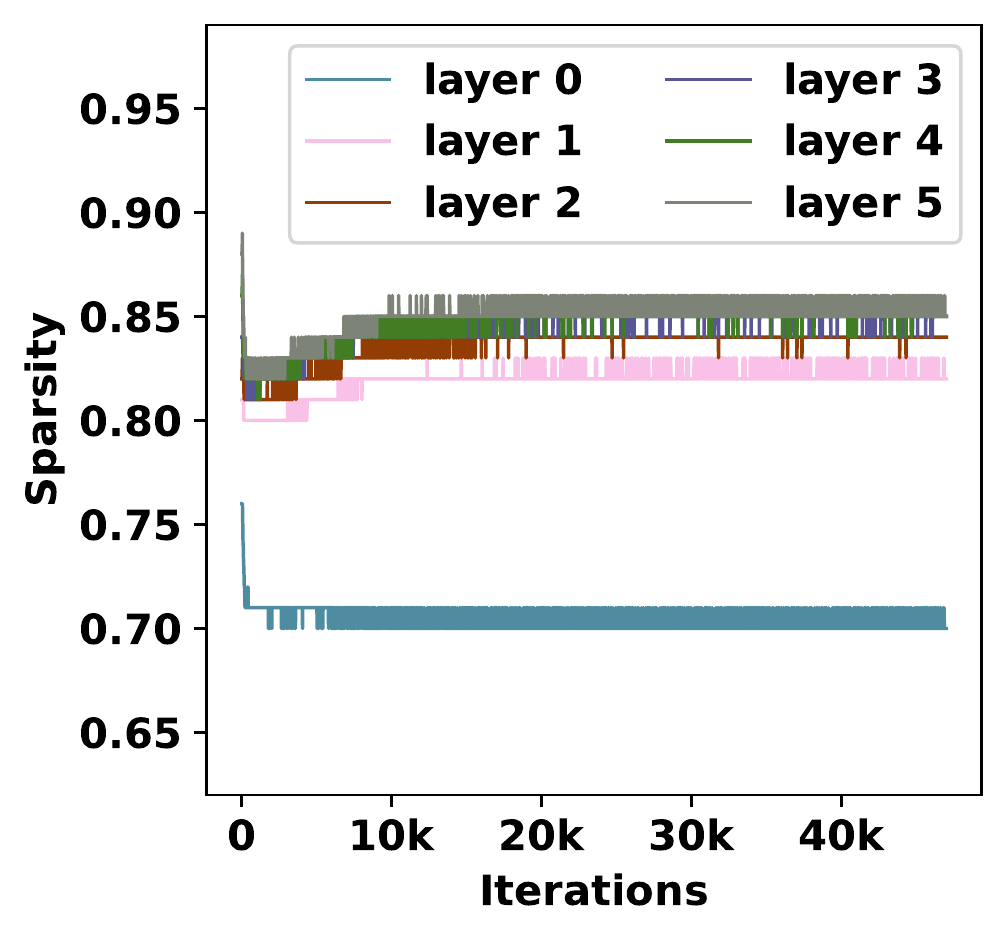}
  \caption{Init Bias as -1.0}
  \label{fig:bias-1.0}
\end{subfigure}
\caption{Sparsity pattern on different layers across different training iterations for three different bias initialization. 
The $x$ and $y$ axis denote the iteration number and sparsity level, respectively. 
The models can achieve $97.9\%, 97.7\%$ and $97.3\%$ accuracy after training, respectively. 
Note that, in Figure (a), the lines of layers 1-5 overlap together except layer 0.} 
  \label{fig:exp}
\end{figure}

\textbf{Observation and Implication.}
As demonstrated at \Cref{fig:exp}, when we initialize the bias with three different values, the sparsity patterns are stable across all layers during training: when the bias is initialized as $0$ and $-0.5$, the sparsity change is within $2.5\%$; and when the bias is initialized as $-1.0$, the sparsity change is within $10\%$. 
Meanwhile, by increasing the initialization magnitude for bias, the sparsity level increases with only marginal accuracy dropping. 
This implies that our theory can be extended to the multi-layer setting (with some extra care for coping with vanishing gradient) and multi-layer neural networks can also benefit from the sparsity-inducing initialization and enjoy reduction of computational cost. 
Another interesting observation is that the input layer (layer 0) has a different sparsity pattern from other layers while all the rest layers behave similarly.

We next provide experiment on the convergence of the network with large bias initialization. We setup a toy example with $x \in \R^5$ and $y \in \R$ where $x \sim \mathcal{N}(0, I)$ and $y = w^\top x$ for fixed $w$ with unit norm.
The network has width 128 and the bias is initialized with $0, -0.1, -0.2, -0.3, -0.4$. We plot the convergence rate in \Cref{fig:training_loss}. As we can see from the plot, all the curves have the same slope at the end of the training, which verifies our claims that network with different bias initialization will have the same convergence rate. 
\begin{figure}
    \centering
    \includegraphics[width=.5\columnwidth]{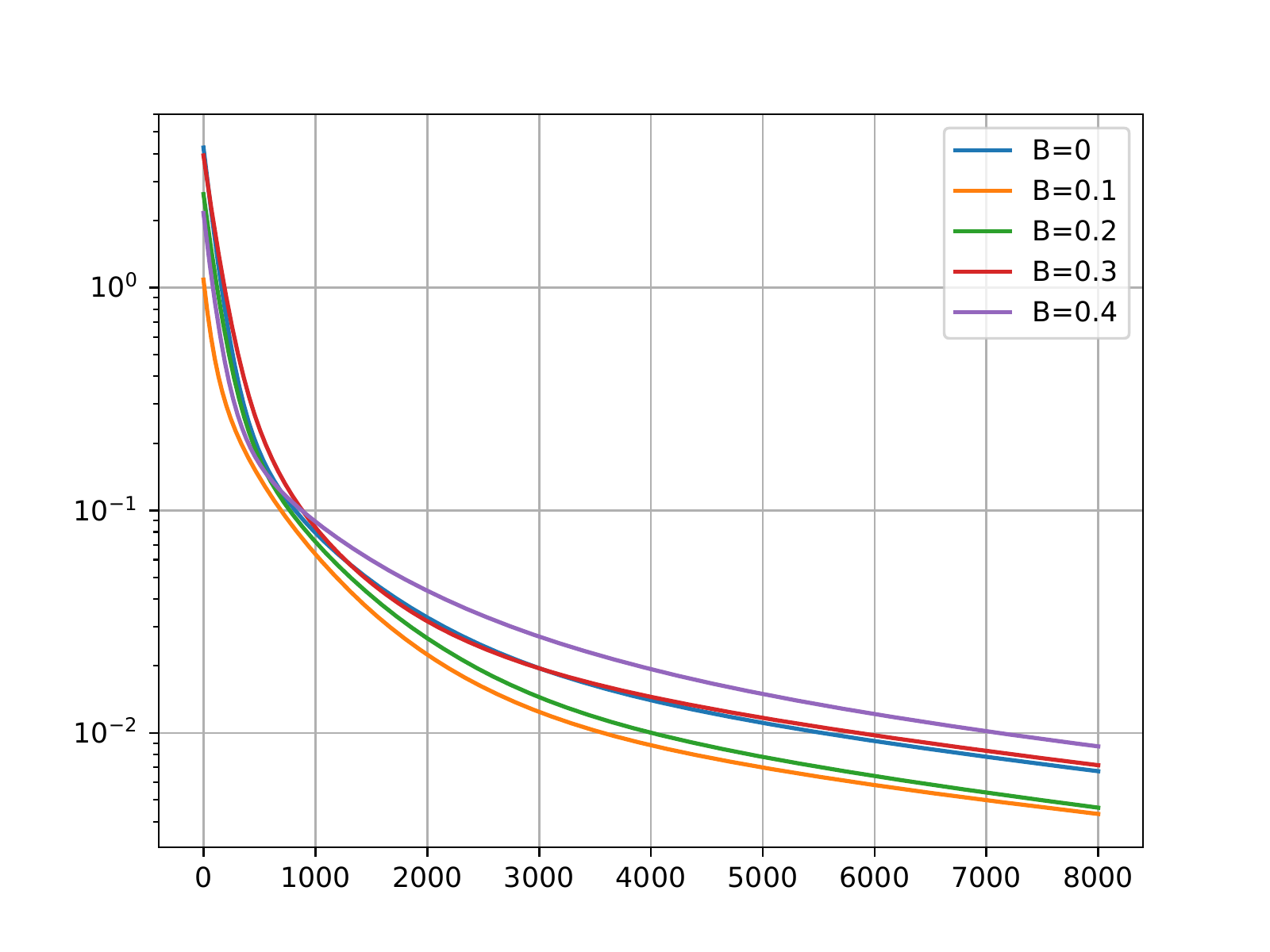}
    \caption{Training loss with different bias initialization.}
    \label{fig:training_loss}
\end{figure}

%% file: appendix.tex
\section{Convergence}\label{app: convergence}
\textbf{Notation simplification.}
Since the smallest eigenvalue of the limiting NTK appeared in this proof all has dependence on the bias initialization parameter $B$, for the ease of notation of our proof, we suppress its dependence on $B$ and use $\lambda$ to denote $\lambda := \lambda(B) = \lambda_{\textnormal{min}}(H^\infty(B))$. 
\subsection{Difference between limit NTK and sampled NTK}
\begin{lemma}\label[lemma]{lemma: diff_discrete_limit_ntk}
For a given bias vector $b \in \R^m$ with $b_r \geq 0,\ \forall r \in [m]$,
the limit NTK $H^{\infty}$ and the sampled NTK $H$ are given as
\begin{align*}
    H^\infty_{ij} &:= \E_{w \sim \mathcal{N}(0, I)} \left[ (\inprod{x_i, x_j} + 1) \mathbb{I}(w_r^\top x_i \geq b_r, w_r^\top x_j \geq b_r) \right], \\
    H_{ij} &:= \frac{1}{m} \sum_{r=1}^m (\inprod{x_i, x_j} + 1) \mathbb{I}(w_r^\top x_i \geq b_r, w_r^\top x_j \geq b_r).
\end{align*}
Let's define $\lambda := \lambda_{\textnormal{min}}(H^\infty)$ and assume $\lambda > 0$. 
If the network width $m = \Omega(\lambda^{-1} n \cdot \log(n/\delta))$, then
\begin{align*}
    \Pr\left[ \lambda_{\textnormal{min}}(H) \geq \frac{3}{4}\lambda \right] \geq 1 - \delta.
\end{align*}
\end{lemma}
\begin{proof}
Let $H_r:= \frac{1}{m} \widetilde{X}(w_r)^\top \widetilde{X}(w_r)$, where $\widetilde{X}(w_r) \in \R^{(d+1) \times n}$ is defined as
\begin{align*}
    \widetilde{X}(w_r) := [\mathbb{I}(w_r^\top x_1 \geq b) \cdot (x_1,1), \ldots, \mathbb{I}(w_r^\top x_n \geq b) \cdot (x_n,1)],
\end{align*}
where $(x_i, 1)$ denotes appending the vector $x_i$ by $1$.
Hence $H_r \succeq 0$. 
Since for each entry $H_{ij}$ we have
\begin{align*}
    (H_r)_{ij} = \frac{1}{m} (\inprod{x_i, x_j} + 1) \mathbb{I}(w_r^\top x_i \geq b_r, w_r^\top x_j \geq b_r) \leq \frac{1}{m} (\inprod{x_i, x_j} + 1) \leq \frac{2}{m},
\end{align*}
and naively, we can upper bound $\norm{H_r}_2$ by:
\begin{align*}
    \norm{H_r}_2 \leq \norm{H_r}_F \leq \sqrt{n^2 \frac{4}{m^2}} = \frac{2n}{m}.
\end{align*}
Then $H = \sum_{r=1}^m H_r$ and $\E[H] = H^\infty$. 
Hence, by the Matrix Chernoff Bound in \Cref{lemma: matrix_chernoff} and choosing $m = \Omega(\lambda^{-1} n \cdot \log(n/\delta))$, we can show that
\begin{align*}
    \Pr\left[ \lambda_{\textnormal{min}}(H) \leq \frac{3}{4}\lambda \right] &\leq n \cdot \exp\left( - \frac{1}{16} \lambda/(4n/m) \right) \\
    &= n \cdot \exp\left( - \frac{\lambda m}{64 n} \right) \\
    &\leq \delta.
\end{align*}
\end{proof}

\begin{lemma}\label{lemma: fro_diff_discrete_limit_ntk}
Assume $m = n^{O(1)}$ and $\exp(B^2/2) = O(\sqrt{m})$ where we recall that $B$ is the initialization value of the biases. 
With probability at least $1 - \delta$, we have $\norm{H(0) - H^\infty}_F \leq 4n\exp(-B^2/4) \sqrt{\frac{\log (n^2/\delta)}{m}}$.
\end{lemma}
\begin{proof}
First, we have $\E[((\inprod{x_i, x_j} + 1) \mathbb{I}_{r,i}(0) \mathbb{I}_{r,j}(0))^2] \leq 4\exp(-B^2/2)$. 
Then, by Bernstein's inequality in \Cref{lemma: bernstein}, with probability at least $1 - \delta/n^2$, 
\begin{align*}
    |H_{ij}(0) - H^\infty_{ij}| \leq 2\exp(-B^2/4) \sqrt{2\frac{\log (n^2/\delta)}{m}} + 2 \frac{2}{m} \log (n^2/\delta) \leq 4\exp(-B^2/4) \sqrt{\frac{\log (n^2/\delta)}{m}}.
\end{align*}
By a union bound, the above holds for all $i, j \in [n]$ with probability at least $1 - \delta$, which implies
\begin{align*}
    \norm{H(0) - H^\infty}_F \leq 4n\exp(-B^2/4) \sqrt{\frac{\log (n^2/\delta)}{m}}.
\end{align*}
\end{proof}

\subsection{Bounding the number of flipped neurons}
\begin{definition}[No-flipping set]\label{def: noflipping_set}
For each $i \in [n]$, let $S_i \subset [m]$ denote the set of neurons that are never flipped during the entire training process,
\begin{align*}
    S_i := \{r \in [m]:\ \forall t \in [T]\ \sign(\inprod{w_r(t), x_i} - b_r(t)) = \sign(\inprod{w_r(0), x_i} - b_r(0))\}.
\end{align*}
Thus, the flipping set is $\overline{S}_i$ for $i \in [n]$. 
\end{definition}

\begin{lemma}[Bound on flipping probability]\label{lemma: bound_flipping}
Let $B \geq 0$ and $R_w, R_b \leq \min\{1/B,1\}$. 
Let $\tilde{W} = (\widetilde{w}_1, \ldots, \tilde{w}_m)$ be vectors generated i.i.d. from $\mathcal{N}(0, I)$ and $\tilde{b} = (\tilde{b}_1, \ldots, \tilde{b}_m)= (B, \ldots, B)$, 
and weights $ W = (w_1, \ldots, w_m)$ and biases $ b = (b_1, \ldots, b_m) $ that satisfy for any $r \in [m]$, $\norm{\tilde{w}_r - w_r}_2 \leq R_w$ and $|\tilde{b}_r - b_r| \leq R_b$.
Define the event
\begin{align*}
    A_{i,r} = \{\exists w_r, b_r: \norm{\tilde{w}_r - w_r}_2 \leq R_w,\ |b_r - \tilde{b}_r| \leq R_b,\ \mathbb{I}(x_i^\top \tilde{w}_r \geq \tilde{b}_r) \neq \mathbb{I}(x_i^\top w_r \geq b_r)\}.
\end{align*}
Then,
\begin{align*}
    \Pr\left[ A_{i,r} \right] \leq c (R_w + R_b) \exp(-B^2/2)
\end{align*}
for some constant $c$. 
\end{lemma}

\begin{proof}
Notice that the event $A_{i,r}$ happens if and only if $|\tilde{w}_r^\top x_i - \tilde{b}_r| < R_w + R_b$. 
First, if $B > 1$, then by \Cref{lemma: gaussian_range_bound}, we have 
\begin{align*}
    \Pr\left[ A_{i,r} \right] \leq (R_w + R_b) \frac{1}{\sqrt{2\pi}} \exp(-(B - R_w - R_b)^2/2) \leq c_1 (R_w + R_b) \exp(-B^2/2)
\end{align*}
for some constant $c_1$. 
If $0 \leq B < 1$, then the above analysis doesn't hold since it is possible that $B - R_w - R_b \leq 0 $. 
In this case, the probability is at most $\Pr[A_{i,r}] \leq 2(R_w + R_b) \frac{1}{\sqrt{2\pi}}\exp(-0^2/2) = \frac{2(R_w + R_b)}{\sqrt{2\pi}}$. 
However, since $0 \leq B < 1$ in this case, we have $\exp(-1^2/2) \leq \exp(-B^2/2) \leq \exp(-0^2/2)$. 
Therefore, $\Pr[A_{i,r}] \leq c_2 (R_w + R_b) \exp(-B^2/2)$ for $c_2 = \frac{2\exp(1/2)}{\sqrt{2\pi}}$. 
Take $c = \max\{c_1, c_2\}$ finishes the proof.
\end{proof}

\begin{corollary}\label{corollary: num_flipped_neurons}
Let $B > 0$ and $R_w, R_b \leq \min\{1/B, 1\}$. 
Assume that $\norm{w_r(t) - w_r(0)}_2 \leq R_w$ and $|b_r(t) - b_r(0)| \leq R_b$ for all $t \in [T]$. 
For $i \in [n]$, the flipping set $\overline{S}_i$ satisfies that
\begin{align*}
    \Pr[r \in \overline{S}_i] \leq c(R_w + R_b)\exp(-B^2/2)
\end{align*}
for some constant $c$, which implies
\begin{align*}
    &\Pr[\forall i \in [n]:\ |\overline{S}_i| \leq 2 m c(R_w + R_b)\exp(-B^2/2)] \\
    &\quad \geq 1 - n \cdot \exp\left(-\frac{2}{3} m c(R_w + R_b)\exp(-B^2/2) \right).
\end{align*}
\end{corollary}
\begin{proof}
The proof is by observing that $\Pr[r \in \overline{S}_i] \leq \Pr[A_{i,r}]$. Then, by Bernstein's inequality, 
\begin{align*}
    \Pr[|\overline{S}_i| > t] \leq \exp\left( - \frac{t^2/2}{m c(R_w + R_b)\exp(-B^2/2) + t/3} \right).
\end{align*}
Take $t = 2 m c(R_w + R_b)\exp(-B^2/2)$ and a union bound over $[n]$, we have
\begin{align*}
    & \Pr[\forall i \in [n]:\ |\overline{S}_i| \leq 2 m c(R_w + R_b)\exp(-B^2/2)] \\
    &\quad \geq 1 - n \cdot \exp\left(-\frac{2}{3} m c(R_w + R_b)\exp(-B^2/2) \right).
\end{align*}
\end{proof}

\subsection{Bounding NTK if perturbing weights and biases}
\begin{lemma}\label[lemma]{lemma: perturbed_ntk}
Assume $\lambda > 0$. 
Let $B>0$ and $R_b, R_w \leq \min\{1/B, 1\}$. 
Let $\tilde{W} = (\widetilde{w}_1, \ldots, \tilde{w}_m)$ be vectors generated i.i.d. from $\mathcal{N}(0, I)$ and $\tilde{b} = (\tilde{b}_1, \ldots, \tilde{b}_m)= (B, \ldots, B)$. 
For any set of weights $ W = (w_1, \ldots, w_m)$ and biases $ b = (b_1, \ldots, b_m) $ that satisfy for any $r \in [m]$, $\norm{\tilde{w}_r - w_r}_2 \leq R_w$ and $|\tilde{b}_r - b_r| \leq R_b$, we define the matrix $H(W,b) \in \R^{n \times n}$ by
\begin{align*}
    H_{ij}(W,b) = \frac{1}{m} \sum_{r=1}^m (\inprod{x_i, x_j} + 1) \mathbb{I}(w_r^\top x_i \geq b_r, w_r^\top x_j \geq b_r).
\end{align*}
It satisfies that for some small positive constant $c$,
\begin{enumerate}
    \item With probability at least $1 - n^2 \exp\left( -\frac{2}{3}c m (R_w + R_b) \exp(-B^2/2) \right)$, we have 
    \begin{align*}
    \norm{H(\tilde{W}, \tilde{b}) - H(W,b)}_F &\leq n \cdot 8c (R_w + R_b) \exp(-B^2/2), \\
    \norm{Z(\tilde{W}, \tilde{b}) - Z(W,b)}_F &\leq \sqrt{n\cdot 8c (R_w + R_b) \exp(-B^2/2)}. 
\end{align*}
    
    \item  With probability at least $1 - \delta - n^2 \exp\left( -\frac{2}{3}c m (R_w + R_b) \exp(-B^2/2) \right)$, 
    \begin{align*}
        \lambda_{\textnormal{min}}(H(W,b)) > 0.75 \lambda - n \cdot 8c (R_w + R_b) \exp(-B^2/2).
    \end{align*}
\end{enumerate}
\end{lemma}
\begin{proof}
We have
\begin{align*}
    \norm{Z(W,b) - Z(\tilde{W}, \tilde{b})}_F^2 &= \sum_{i \in [n]} \left( \frac{2}{m} \sum_{r \in [m]} \left( \mathbb{I}(w_r^\top x_i \geq b_r) - \mathbb{I}(\tilde{w}_r^\top x_i \geq \tilde{b}_r) \right)^2 \right) \\
    &= \sum_{i \in [n]} \left( \frac{2}{m} \sum_{r \in [m]} t_{r,i} \right)
\end{align*}
and
\begin{align*}
    & \norm{H(W,b) - H(\tilde{W}, \tilde{b})}_F^2 \\
    &= \sum_{i \in [n],\ j \in [n]} (H_{ij}(W,b) - H_{ij}(\tilde{W}, \tilde{b}))^2 \\
    &\leq \frac{4}{m^2} \sum_{i \in [n],\ j \in [n]} \left( \sum_{r \in [m]} | \mathbb{I}(w_r^\top x_i \geq b_r, w_r^\top x_j \geq b_r) - \mathbb{I}(\tilde{w}_r^\top x_i \geq \tilde{b}_r, \tilde{w}_r^\top x_j \geq \tilde{b}_r) | \right)^2 \\
    &= \frac{4}{m^2} \sum_{i,j \in [n]} \left( \sum_{r \in [m]} s_{r,i,j} \right)^2,
\end{align*}
where we define 
\begin{align*}
    s_{r,i,j} &:= |\mathbb{I}(w_r^\top x_i \geq b_r, w_r^\top x_j \geq b_r) - \mathbb{I}(\tilde{w}_r^\top x_i \geq \tilde{b}_r, \tilde{w}_r^\top x_j \geq \tilde{b}_r)| ,\\
    t_{r,i} &:= (\mathbb{I}(w_r^\top x_i \geq b_r) - \mathbb{I}(\tilde{w}_r^\top x_i \geq \tilde{b}_r))^2.
\end{align*}
Notice that $t_{r,i} = 1$ only if the event $A_{i,r}$ happens (recall the definition of $A_{i,r}$ in \Cref{lemma: bound_flipping}) and $s_{r,i,j} = 1$ only if the event $A_{i,r}$ or $A_{j,r}$ happens. 
Thus, 
\begin{align*}
    \sum_{r \in [m]} t_{r,i} \leq \sum_{r \in [m]} \mathbb{I}(A_{i,r}), \quad \sum_{r \in [m]} s_{r,i,j} \leq \sum_{r \in [m]} \mathbb{I}(A_{i,r}) + \mathbb{I}(A_{j,r}).
\end{align*}
By \Cref{lemma: bound_flipping}, we have
\begin{align*}
    \E_{\tilde{w}_r}[s_{r,i,j}] \leq \E_{\tilde{w}_r}[s_{r,i,j}^2] \leq \Pr_{\tilde{w}_r}[A_{i,r}] + \Pr_{\tilde{w}_r}[A_{j,r}] \leq 2c (R_w + R_b) \exp(-B^2/2).
\end{align*}
Define $s_{i,j} = \sum_{r=1}^m \mathbb{I}(A_{i,r}) + \mathbb{I}(A_{j,r})$. 
By Bernstein's inequality in \Cref{lemma: bernstein}, 
\begin{align*}
    & \Pr\left[ s_{i,j} \geq m \cdot 2c (R_w + R_b) \exp(-B^2/2) + mt \right] \\
    &\leq \exp \left( - \frac{m^2 t^2/2}{m \cdot 2c (R_w + R_b) \exp(-B^2/2) + mt/3} \right), \quad \forall t \geq 0.
\end{align*}
Let $t = 2c(R_w + R_b) \exp(-B^2/2)$. We get
\begin{align*}
    \Pr[s_{i,j} \geq m \cdot 4c(R_w + R_b) \exp(-B^2/2)] \leq \exp\left( -\frac{2}{3}c m (R_w + R_b) \exp(-B^2/2) \right).
\end{align*}
Thus, we obtain with probability at least $1 - n^2 \exp\left( -\frac{2}{3}c m (R_w + R_b) \exp(-B^2/2) \right)$,
\begin{align*}
    \norm{H(\tilde{W}, \tilde{b}) - H(W,b)}_F &\leq n \cdot 8c (R_w + R_b) \exp(-B^2/2) ,\\
    \norm{Z(\tilde{W}, \tilde{b}) - Z(W,b)}_F &\leq \sqrt{n\cdot 8c (R_w + R_b) \exp(-B^2/2)} .
\end{align*}
For the second result, by \Cref{lemma: diff_discrete_limit_ntk}, $\Pr[\lambda_{\textnormal{min}}(H(\tilde{W}, \tilde{b})) \geq 0.75 \lambda] \geq 1 - \delta$. Hence, with probability at least $1 - \delta - n^2 \exp\left( -\frac{2}{3}c m (R_w + R_b) \exp(-B^2/2) \right)$,
\begin{align*}
    \lambda_{\textnormal{min}}(H(W, b)) &\geq \lambda_{\textnormal{min}}(H(\tilde{W}, \tilde{b})) - \norm{H(W, b) - H(\tilde{W}, \tilde{b})} \\
    &\geq \lambda_{\textnormal{min}}(H(\tilde{W}, \tilde{b})) - \norm{H(W, b) - H(\tilde{W}, \tilde{b})}_F \\
    &\geq 0.75 \lambda - n \cdot 8c (R_w + R_b) \exp(-B^2/2).
\end{align*}
\end{proof}

\subsection{Total movement of weights and biases}
\begin{definition}[NTK at time $t$]
For $t \geq 0$, let $H(t)$ be an $n \times n$ matrix with $(i,j)$-th entry
\begin{align*}
    H_{ij}(t) &:= \inprod{\frac{\partial f(x_i; \theta(t))}{\partial \theta(t)}, \frac{\partial f(x_j; \theta(t))}{\partial \theta(t)}} \\
    &= \frac{1}{m} \sum_{r=1}^m (\inprod{x_i, x_j} + 1) \mathbb{I}(w_r(t)^\top x_i \geq b_r(t), w_r(t)^\top x_j \geq b_r(t)).
\end{align*}
\end{definition}
We follow the proof strategy from \citep{du2018gradient}. 
Now we derive the total movement of weights and biases. 
Let $f(t) = f(X; \theta(t))$ where $f_i(t) = f(x_i; \theta(t))$. 
The dynamics of each prediction is given by
\begin{align*}
    \frac{d}{dt}f_i(t) &= \inprod{\frac{\partial f(x_i; \theta(t))}{\partial \theta(t)}, \frac{d \theta(t)}{dt}} \\
    &= \sum_{j=1}^n (y_j - f_j(t)) \inprod{\frac{\partial f(x_i; \theta(t))}{\partial \theta(t)}, \frac{\partial f(x_j; \theta(t))}{\partial \theta(t)}} \\ 
    &= \sum_{j=1}^n (y_j - f_j(t)) H_{ij}(t),
\end{align*}
which implies
\begin{align}\label{eq: prediction_dynamics}
    \frac{d}{dt} f(t) = H(t)(y - f(t)).
\end{align}

\begin{lemma}[Gradient Bounds]\label{lemma: gradient_bound}
For any $0 \leq s \leq t$, we have
\begin{align*}
    \norm{\frac{\partial L(W(s), b(s))}{\partial w_r(s)}}_2 &\leq \sqrt{\frac{n}{m}} \norm{f(s) - y}_2, \\
    \norm{\frac{\partial L(W(s), b(s))}{\partial b_r(s)}}_2 &\leq \sqrt{\frac{n}{m}} \norm{f(s) - y}_2.
\end{align*}
\end{lemma}
\begin{proof}
We have:
\begin{align*}
    \norm{\frac{\partial L(W(s), b(s))}{\partial w_r(s)}}_2 &= \norm{\frac{1}{\sqrt{m}} \sum_{i=1}^n (f(x_i; W(s), b(s)) - y_i) a_r x_i \mathbb{I}(w_r(s)^\top x_i \geq b_r)}_2 \\
    &\leq \frac{1}{\sqrt{m}} \sum_{i=1}^n |f(x_i; W(s), b(s)) - y_i| \\
    &\leq \sqrt{\frac{n}{m}} \norm{f(s) - y}_2,
\end{align*}
where the first inequality follows from triangle inequality, and the second inequality follows from Cauchy-Schwarz inequality.

Similarly, we also have:
\begin{align*}
    \norm{\frac{\partial L(W(s), b(s))}{\partial b_r(s)}}_2 &= \norm{\frac{1}{\sqrt{m}} \sum_{i=1}^n (f(x_i; W(s), b(s)) - y_i) a_r \mathbb{I}(w_r(s)^\top x_i \geq b_r)}_2 \\
    &\leq \frac{1}{\sqrt{m}} \sum_{i=1}^n |f(x_i; W(s), b(s)) - y_i| \\
    &\leq \sqrt{\frac{n}{m}} \norm{f(s) - y}_2.
\end{align*}
\end{proof}

\subsubsection{Gradient Descent}
\begin{lemma}\label[lemma]{lemma: weight_bias_movement}
Assume $\lambda >0$. 
Assume $\norm{y - f(k)}_2^2 \leq (1 - \eta \lambda / 4)^k \norm{y - f(0)}_2^2$ holds for all $k' \leq k$. 
Then for every $r \in [m]$, 
\begin{align*}
    \norm{w_r(k+1) - w_r(0)}_2 &\leq \frac{8\sqrt{n} \norm{y - f(0)}_2}{\sqrt{m} \lambda} := D_w, \\
    | b_r(k+1) - b_r(0) | &\leq \frac{8\sqrt{n} \norm{y - f(0)}_2}{\sqrt{m} \lambda} := D_b.
\end{align*}
\end{lemma}
\begin{proof}
\begin{align*}
    \norm{w_r(k+1) - w_r(0)}_2 &\leq \eta \sum_{k'=0}^k \norm{\frac{\partial L(W(k'))}{\partial w_r(k')}}_2 \\
    &\leq \eta \sum_{k' = 0}^k \sqrt{\frac{n}{m}} \norm{y - f(k')}_2 \\
    &\leq \eta \sum_{k' = 0}^k \sqrt{\frac{n}{m}} (1 - \eta \lambda / 4)^{k'/2} \norm{y - f(0)}_2 \\
    &\leq \eta \sum_{k' = 0}^k \sqrt{\frac{n}{m}} (1 - \eta \lambda / 8)^{k'} \norm{y - f(0)}_2 \\
    &\leq \eta \sum_{k' = 0}^\infty \sqrt{\frac{n}{m}} (1 - \eta \lambda / 8)^{k'} \norm{y - f(0)}_2 \\
    &\leq  \frac{8\sqrt{n}}{\sqrt{m} \lambda} \norm{y - f(0)}_2,
\end{align*}
where the first inequality is by Triangle inequality, the second inequality is by \Cref{lemma: gradient_bound}, the third inequality is by our assumption and the fourth inequality is by $(1 - x)^{1/2} \leq 1 - x/2$ for $x \geq 0$. 

The proof for $b$ is similar. 
\end{proof}

\subsubsection{Gradient Flow}
\begin{lemma}\label{lemma: weight_bias_movement_gf}
Suppose for $0 \leq s \leq t$, $\lambda_{\textnormal{min}}(H(s)) \geq \frac{\lambda_0}{2} > 0$. 
Then we have $\norm{y - f(t)}_2^2 \leq \exp(-\lambda_0 t) \norm{y - f(0)}_2^2$ and for any $r \in [m]$, $\norm{w_r(t) - w_r(0)}_2 \leq \frac{\sqrt{n} \norm{y - f(0)}_2}{\sqrt{m} \lambda_0}$ and $|b_r(t) - b_r(0)| \leq \frac{\sqrt{n} \norm{y - f(0)}_2}{\sqrt{m} \lambda_0}$. 
\end{lemma}
\begin{proof}
By the dynamics of prediction in \Cref{eq: prediction_dynamics}, we have
\begin{align*}
    \frac{d}{dt} \norm{y - f(t)}_2^2 &= -2(y - f(t))^\top H(t) (y - f(t)) \\
    &\leq -\lambda_0 \norm{y - f(t)}_2^2,
\end{align*}
which implies
\begin{align*}
    \norm{y - f(t)}_2^2 \leq \exp(-\lambda_0 t) \norm{y - f(t)}_2^2.
\end{align*}
Now we bound the gradient norm of the weights
\begin{align*}
    \norm{\frac{d}{ds} w_r(s)}_2 &= \norm{\sum_{i=1}^n (y_i - f_i(s)) \frac{1}{\sqrt{m}} a_r x_i \mathbb{I}(w_r(s)^\top x_i \geq b(s))}_2 \\
    &\leq \frac{1}{\sqrt{m}} \sum_{i=1}^n |y_i f_i(s)| \leq \frac{\sqrt{n}}{\sqrt{m}} \norm{y - f(s)}_2 \leq \frac{\sqrt{n}}{\sqrt{m}} \exp(-\lambda_0 s) \norm{y - f(0)}_2.
\end{align*}
Integrating the gradient, the change of weight can be bounded as
\begin{align*}
    \norm{w_r(t) - w_r(0)}_2 \leq \int_0^t \norm{\frac{d}{ds} w_r(s)}_2 ds \leq \frac{\sqrt{n} \norm{y - f(0)}_2}{\sqrt{m} \lambda_0}.
\end{align*}
For bias, we have
\begin{align*}
    \norm{\frac{d}{ds} b_r(s)}_2 &= \norm{\sum_{i=1}^n (y_i - f_i(s)) \frac{1}{\sqrt{m}} a_r \mathbb{I}(w_r(s)^\top x_i \geq b(s))}_2 \\
    &\leq \frac{1}{\sqrt{m}} \sum_{i=1}^n |y_i -f_i(s)| \leq \frac{\sqrt{n}}{\sqrt{m}} \norm{y - f(s)}_2 \leq \frac{\sqrt{n}}{\sqrt{m}} \exp(-\lambda_0 s) \norm{y - f(0)}_2.
\end{align*}
Now, the change of bias can be bounded as
\begin{align*}
    \norm{b_r(t) - b_r(0)}_2 \leq \int_0^t \norm{\frac{d}{ds} w_r(s)}_2 ds \leq \frac{\sqrt{n} \norm{y - f(0)}_2}{\sqrt{m} \lambda_0}.
\end{align*}
\end{proof}

\subsection{Gradient Descent Convergence Analysis}
\subsubsection{Upper bound of the initial error}
\begin{lemma}[Initial error upper bound]\label{claim: initial_error}
Let $B > 0$ be the initialization value of the biases and all the weights be initialized from standard Gaussian. 
Let $\delta \in (0,1)$ be the failure probability. Then, with probability at least $1 - \delta$, we have
\begin{align*}
    \norm{f(0)}_2^2 &= O(n (\exp(-B^2/2) + 1/m) \log^3(mn/\delta)) ,\\
    \norm{f(0) - y}_2^2 &= O\left(n + n \left( \exp(-B^2/2) + {1/m} \right) \log^{3} (2mn/\delta) \right).
\end{align*}
\end{lemma}
\begin{proof}
Since we are only analyzing the initialization stage, for notation ease, we omit the dependence on time without any confusion. 
We compute
\begin{align*}
    \norm{y - f}_2^2 &= \sum_{i=1}^n (y_i - f(x_i))^2 \\
    &= \sum_{i=1}^n \left( y_i - \frac{1}{\sqrt{m}} \sum_{r=1}^m a_r \sigma(w_r^\top x_i - B) \right)^2 \\
    &= \sum_{i=1}^n \left( y_i^2 - 2 \frac{y_i}{\sqrt{m}} \sum_{r=1}^m a_r \sigma(w_r^\top x_i - B) + \frac{1}{m} \left( \sum_{r=1}^m a_r \sigma(w_r^\top x_i - B) \right)^2 \right) .
\end{align*}
Since $w_r^\top x_i \sim \mathcal{N}(0,1)$ for all $r \in [m]$ and $i \in [n]$, by Gaussian tail bound and a union bound over $r,i$, we have
\begin{align*}
    \Pr[\forall i \in [n],\ j \in [m]: w_r^\top x_i \leq \sqrt{2\log(2mn/\delta)}] \geq 1 - \delta/2.
\end{align*}
Let $E_1$ denote this event. 
Conditioning on the event $E_1$, let 
\begin{align*}
    z_{i,r} := \frac{1}{\sqrt{m}} \cdot a_r \cdot \min\left\{ \sigma(w_r^\top x_i - B), \sqrt{2\log(2mn/\delta)} \right\}.
\end{align*}

Notice that $z_{i,r} \neq 0$ with probability at most $\exp(-B^2/2)$. Thus,
\begin{align*}
    \E_{a_r, w_r}[z_{i,r}^2] &\leq \exp(-B^2/2) \frac{1}{m} 2\log(2mn/\delta).
\end{align*}
By randomness in $a_r$, we know $\E[z_{i,r}] = 0$. 
Now apply Bernstein's inequality in \Cref{lemma: bernstein}, we have for all $t > 0$, 
\begin{align*}
    \Pr\left[ \left| \sum_{r=1}^m z_{i,r} \right| > t \right] \leq \exp\left( - \min\left( \frac{t^2/2}{4 \exp(-B^2/2) \log(2mn/\delta)}, \frac{\sqrt{m} t/2}{2 \sqrt{2\log(2mn/\delta)}} \right) \right).
\end{align*}
Thus, by a union bound, with probability at least $1 - \delta/2$, for all $i \in [n]$,
\begin{align*}
    \left| \sum_{r=1}^m z_{i,r} \right| &\leq \sqrt{2\log(2mn/\delta) \exp(-B^2/2) 2\log(2n/\delta)} + 2 \sqrt{\frac{2\log(2mn/\delta)}{m}} \log ({2n}/{\delta}) \\
    &\leq \left( 2\exp(-B^2/4) + 2\sqrt{2/m} \right) \log^{3/2} (2mn/\delta).
\end{align*}
Let $E_2$ denote this event. 
Thus, conditioning on the events $E_1, E_2$, with probability $1 - \delta$, 
\begin{align*}
   \norm{f(0)}_2^2 &= \sum_{i=1}^n \left( \sum_{r=1}^m z_{i,r} \right)^2 = O(n (\exp(-B^2/2) + 1/m) \log^3(mn/\delta))
\end{align*}
and
\begin{align*}
    & \norm{y - f(0)}_2^2 \\
    &= \sum_{i=1}^n y_i^2 - 2 \sum_{i=1}^n y_i \sum_{r=1}^m z_{i,r} + \sum_{i=1}^n \left( \sum_{r=1}^m z_{i,r} \right)^2 \\
    &\leq \sum_{i=1}^n y_i^2 + 2 \sum_{i=1}^n |y_i| \left( 2\exp(-B^2/4) + 2\sqrt{2/m} \right) \log^{3/2} (2mn/\delta) \\
    &\quad + \sum_{i=1}^n \left( \left( 2\exp(-B^2/4) + 2\sqrt{2/m} \right) \log^{3/2} (2mn/\delta) \right)^2 \\
    &= O\left(n + n \left( \exp(-B^2/2) + {1/m} \right) \log^{3} (2mn/\delta) \right),
\end{align*}
where we assume $y_i = O(1)$ for all $i \in [n]$.
\end{proof}

\subsubsection{Error Decomposition}
We follow the proof outline in \citep{song2019quadratic, song2021does} and we generalize it to networks with trainable $b$. 
Let us define matrix $H^\bot$ similar to $H$ except only considering flipped neurons by
\begin{align*}
    H^\bot_{ij}(k) := \frac{1}{m} \sum_{r \in \overline{S}_i}  (\inprod{x_i, x_j} + 1) \mathbb{I}(w_r(k)^\top x_i \geq b_r(k), w_r(k)^\top x_j \geq b_r(k))
\end{align*}
and vector $v_1, v_2$ by
\begin{align*}
    v_{1,i} := \frac{1}{\sqrt{m}} \sum_{r \in S_i} a_r(\sigma(\inprod{w_r(k+1), x_i} - b_r(k+1)) - \sigma(\inprod{w_r(k), x_i} - b_r(k))) ,\\
    v_{2,i} := \frac{1}{\sqrt{m}} \sum_{r \in \overline{S}_i} a_r(\sigma(\inprod{w_r(k+1), x_i} - b_r(k+1)) - \sigma(\inprod{w_r(k), x_i} - b_r(k))).
\end{align*}
Now we give out our error update. 
\begin{claim}
\begin{align*}
    \norm{y - f(k+1)}_2^2 = \norm{y - f(k)}_2^2 + B_1 + B_2 + B_3 + B_4,
\end{align*}
where 
\begin{align*}
    B_1 &:= -2\eta (y - f(k))^\top H(k) (y - f(k)) ,\\
    B_2 &:= 2\eta (y - f(k))^\top H^\bot(k) (y - f(k)) ,\\
    B_3 &:= - 2(y - f(k))^\top v_2 ,\\
    B_4 &:= \norm{f(k+1) - f(k)}_2^2 .
\end{align*}
\end{claim}
\begin{proof}
First we can write
\begin{align*}
    v_{1,i} &= \frac{1}{\sqrt{m}} \sum_{r \in S_i} a_r \left( \sigma\left( \inprod{w_r(k) - \eta \frac{\partial L}{\partial w_r}, x_i} - \left( b_r(k) - \eta \frac{\partial L}{\partial b_r} \right) \right) - \sigma(\inprod{w_r(k), x_i} - b_r(k)) \right) \\
    &= \frac{1}{\sqrt{m}} \sum_{r \in S_i} a_r \left( \inprod{-\eta \frac{\partial L}{\partial w_r}, x_i} + \eta \frac{\partial L}{\partial b_r} \right) \mathbb{I}(\inprod{w_r(k), x_i} - b_r(k) \geq 0) \\
    &= \frac{1}{\sqrt{m}} \sum_{r \in S_i} a_r \left( \eta \frac{1}{\sqrt{m}} \sum_{j=1}^n (y_j - f_j(k)) a_r (\inprod{x_j, x_i}+1) \mathbb{I}(w_r(k)^\top x_j \geq b_r(k)) \right) \\
    &\quad \cdot \mathbb{I}(\inprod{w_r(k), x_i} - b_r(k) \geq 0) \\
    &= {\eta} \sum_{j=1}^n (y_j - f_j(k)) (H_{ij}(k) - H^\bot_{ij}(k))
\end{align*}
which means 
\begin{align*}
    v_1 = \eta (H(k) - H^\bot(k))(y - f(k)) .
\end{align*}
Now we compute
\begin{align*}
    \norm{y - f(k+1)}_2^2 &= \norm{y - f(k) - (f(k+1) - f(k))}_2^2 \\
    &= \norm{y - f(k)}_2^2 - 2(y - f(k))^\top (f(k+1) - f(k)) + \norm{f(k+1) - f(k)}_2^2.
\end{align*}
Since $f(k+1) - f(k) = v_1 + v_2$, we can write the cross product term as
\begin{align*}
    & (y - f(k))^\top (f(k+1) - f(k)) \\
    &= (y - f(k))^\top (v_1 + v_2) \\
    &= (y - f(k))^\top v_1 + (y - f(k))^\top v_2 \\
    &= \eta (y - f(k))^\top H(k) (y - f(k)) \\
    & \quad - \eta(y - f(k))^\top H^\bot(k) (y - f(k)) + (y - f(k))^\top v_2.
\end{align*}
\end{proof}

\subsubsection{Bounding the decrease of the error}
\begin{lemma}\label{lemma: error_b1}
Assume $\lambda>0$. 
Assume we choose $R_w, R_b, B$ where $R_w, R_b \leq \min\{1/B, 1\}$ such that $ 8cn(R_w + R_b) \exp(-B^2/2) \leq \lambda/8$. 
Denote $\delta_0 = \delta + n^2 \exp(-\frac{2}{3} cm (R_w + R_b) \exp(-B^2/2))$. 
Then,
\begin{align*}
    \Pr[B_1 \leq -\eta 5\lambda \norm{y - f(k)}_2^2 / 8] \geq 1 - \delta_0.
\end{align*}
\end{lemma}
\begin{proof}
By \Cref{lemma: perturbed_ntk} and our assumption, 
\begin{align*}
    \lambda_{\textnormal{min}}(H(W)) > 0.75 \lambda - n \cdot 8c (R_w + R_b) \exp(-B^2/2) \geq 5\lambda/8
\end{align*}
with probability at least $1 - \delta_0$. Thus,
\begin{align*}
    (y - f(k))^\top H(k) (y - f(k)) \geq \norm{y - f(k)}_2^2 5\lambda / 8.
\end{align*}
\end{proof}

\subsubsection{Bounding the effect of flipped neurons}
Here we bound the term $B_2, B_3$. First, we introduce a fact.
\begin{fact}\label{fact: norm_H_bot}
\begin{align*}
    \norm{H^\bot(k)}_F^2 \leq \frac{4n}{m^2} \sum_{i=1}^n |\overline{S}_i|^2.
\end{align*}
\end{fact}
\begin{proof}
\begin{align*}
    \norm{H^\bot(k)}_F^2 &= \sum_{i,j \in [n]} \left( \frac{1}{m} \sum_{r \in \overline{S}_i} (x_i^\top x_j + 1) \mathbb{I}(w_r(k)^\top x_i \geq b_r(k),\ w_r(k)^\top x_j \geq b_r(k)) \right)^2 \\
    &\leq \sum_{i,j \in [n]} \left( \frac{1}{m} 2 |\overline{S}_i| \right)^2 \leq \frac{4n}{m^2} \sum_{i=1}^n |\overline{S}_i|^2.
\end{align*}
\end{proof}

\begin{lemma}\label{lemma: error_b2}
Denote $\delta_0 = n \exp(-\frac{2}{3} cm (R_w + R_b) \exp(-B^2/2))$. 
Then, 
\begin{align*}
    \Pr[B_2 \leq 8\eta n c(R_w + R_b)\exp(-B^2/2) \cdot \norm{y - f(k)}_2^2] \geq 1 - \delta_0.
\end{align*}
\end{lemma}
\begin{proof}
First, we have 
\begin{align*}
    B_2 \leq 2 \eta \norm{y - f(k)}_2^2 \norm{H^\bot(k)}_2.
\end{align*}
Then, by \Cref{fact: norm_H_bot},
\begin{align*}
    \norm{H^\bot(k)}_2^2 \leq \norm{H^\bot(k)}_F^2 \leq \frac{4n}{m^2} \sum_{i=1}^n |\overline{S}_i|^2.
\end{align*}
By \Cref{corollary: num_flipped_neurons}, we have
\begin{align*}
    \Pr[\forall i \in [n]:\ |\overline{S}_i| \leq 2 m c(R_w + R_b)\exp(-B^2/2)] \geq 1 - \delta_0.
\end{align*}
Thus, with probability at least $1 - \delta_0$, 
\begin{align*}
    \norm{H^\bot(k)}_2 \leq 4 n c(R_w + R_b)\exp(-B^2/2).
\end{align*}
\end{proof}

\begin{lemma}\label{lemma: error_b3}
Denote $\delta_0 = n \exp(-\frac{2}{3} cm (R_w + R_b) \exp(-B^2/2))$. 
Then, 
\begin{align*}
    \Pr[B_3 \leq 4c \eta n (R_w + R_b) \exp(-B^2/2)  \norm{y - f(k)}_2^2] \geq 1 - \delta_0.
\end{align*}
\end{lemma}
\begin{proof}
By Cauchy-Schwarz inequality, we have $B_3 \leq 2 \norm{y - f(k)}_2 \norm{v_2}_2$. 
We have
\begin{align*}
    \norm{v_2}_2^2 &\leq \sum_{i=1}^n \left( \frac{\eta}{\sqrt{m}} \sum_{r \in \overline{S}_i} \left| \inprod{\frac{\partial L}{\partial w_r}, x_i} \right| + \left| \frac{\partial L}{\partial b_r} \right| \right)^2 \\
    &\leq \sum_{i=1}^n \frac{\eta^2}{m} \max_{i \in [n]} \left( \left| \inprod{\frac{\partial L}{\partial w_r}, x_i} \right| + \left| \frac{\partial L}{\partial b_r} \right| \right)^2 |\overline{S}_i|^2 \\
    &\leq n \frac{\eta^2}{m} \left( 2\sqrt{\frac{n}{m}} \norm{f(k) - y}_2 2 m c(R_w + R_b)\exp(-B^2/2) \right)^2 \\
    &= 16c^2 \eta^2 n^2 \norm{y - f(k)}_2^2 (R_w + R_b)^2 \exp(-B^2),
\end{align*}
where the last inequality is by \Cref{lemma: gradient_bound} and \Cref{corollary: num_flipped_neurons} which holds with probability at least $1 - \delta_0$. 
\end{proof}

\subsubsection{Bounding the network update}
\begin{lemma}\label[lemma]{lemma: error_b4}

\begin{align*}
    B_4 \leq C_2^2 \eta^2 n^2 \norm{y - f(k)}_2^2 \exp(-B^2).
\end{align*}
for some constant $C_2$.
\end{lemma}
\begin{proof}
Recall that the definition that $\mathcal{S}_{\textnormal{on}}(i,t) = \{r \in [m]: w_r(t)^\top x_i \geq b_r(t)\}$, i.e., the set of neurons that activates for input $x_i$ at the $t$-th step of gradient descent.  
\begin{align*}
    \norm{f(k+1) - f(k)}_2^2 &\leq \sum_{i=1}^n \left( \frac{\eta}{\sqrt{m}} \sum_{r: r \in \mathcal{S}_{\textnormal{on}}(i, k+1) \cup \mathcal{S}_{\textnormal{on}}(i, k)} \left| \inprod{\frac{\partial L}{\partial w_r}, x_i} \right| + \left| \frac{\partial L}{\partial b_r} \right| \right)^2 \\
    &\leq n \frac{\eta^2}{m} (|\mathcal{S}_{\textnormal{on}}(i, k+1)| + |\mathcal{S}_{\textnormal{on}}(i, k)|)^2 \max_{i \in [n]} \left( \left| \inprod{\frac{\partial L}{\partial w_r}, x_i} \right| + \left| \frac{\partial L}{\partial b_r} \right| \right)^2 \\
    &\leq n \frac{\eta^2}{m} \left( C_2 m \exp(-B^2/2) \cdot  \sqrt{\frac{n}{m}} \norm{y - f(k)}_2 \right)^2 \\
    &\leq C_2^2 \eta^2 n^2 \norm{y - f(k)}_2^2 \exp(-B^2).
\end{align*}
where the third inequality is by \Cref{lemma: number_activated_neuron_init} for some $C_2$. 
\end{proof}

\subsubsection{Putting it all together}
\begin{theorem}[Convergence]\label{lemma: convergence}
Assume $\lambda>0$. 
Let $\eta \leq \frac{\lambda \exp(B^2)}{5C_2^2 n^2}$, $B \in [0, \sqrt{0.5 \log m}]$ and 
\begin{align*}
    m \geq \widetilde{\Omega}\left( \lambda^{-4} n^4 \left( 1 + \left( \exp(-B^2/2) + {1/m} \right) \log^{3} (2mn/\delta) \right) \exp(-B^2) \right).
\end{align*}
Assume $\lambda = \lambda_0 \exp(-B^2/2)$ for some constant $\lambda_0$. 
Then,
\begin{align*}
    \Pr\left[\forall t: \norm{y - f(t)}_2^2 \leq (1 - \eta \lambda / 4)^t \norm{y - f(0)}_2^2 \right] \geq 1 - \delta - e^{-\Omega(n)}.
\end{align*}
\end{theorem}
\begin{proof}
From \Cref{lemma: error_b1}, \Cref{lemma: error_b2}, \Cref{lemma: error_b3} and \Cref{lemma: error_b4}, we know with probability at least $1 - 2n^2 \exp(-\frac{2}{3} cm (R_w + R_b) \exp(-B^2/2)) - \delta$, we have
\begin{align*}
    &\norm{y - f(k+1)}_2^2 \\
    &\leq \norm{y - f(k)}_2^2 (1 - 5\eta \lambda / 8 + 12\eta n c(R_w + R_b)\exp(-B^2/2) + C_2^2 \eta^2 n^2 \norm{y - f(k)}_2^2 \exp(-B^2)).
\end{align*}
By \Cref{lemma: weight_bias_movement}, we need 
\begin{align*}
    D_w &= \frac{8\sqrt{n} \norm{y - f(0)}_2}{\sqrt{m} \lambda} \leq R_w ,\\
    D_b &= \frac{8\sqrt{n} \norm{y - f(0)}_2}{\sqrt{m} \lambda} \leq R_b.
\end{align*}
By \Cref{claim: initial_error}, we have
\begin{align*}
    \Pr[\norm{f(0) - y}_2^2 = O\left(n + n \left( \exp(-B^2/2) + {1/m} \right) \log^{3} (2mn/\delta) \right)] \geq 1 - \delta.
\end{align*}
Let $R = \min\{R_w, R_b\},\ D = \max\{D_w, D_b\}$. 
Combine the results we have
\begin{align*}
    R > \Omega(\lambda^{-1} m^{-1/2} n \sqrt{1 + \left( \exp(-B^2/2) + {1/m} \right) \log^{3} (2mn/\delta)}).
\end{align*}
\Cref{lemma: error_b1} requires 
\begin{align*}
    & 8cn(R_w + R_b) \exp(-B^2/2) \leq \lambda/8 \\
    &\Rightarrow R \leq \frac{\lambda \exp(B^2/2)}{128cn}.
\end{align*}
which implies a lower bound on $m$
\begin{align*}
    m \geq \Omega\left( \lambda^{-4} n^4 \left( 1 + \left( \exp(-B^2/2) + {1/m} \right) \log^{3} (2mn/\delta) \right) \exp(-B^2) \right).
\end{align*}
\Cref{lemma: diff_discrete_limit_ntk} further requires a lower bound of $m = \Omega(\lambda^{-1} n \cdot \log(n/\delta))$ which can be ignored. 

\Cref{lemma: perturbed_ntk} further requires $R < \min\{1/ B, 1\}$ which implies
\begin{align*}
    B &< \frac{128cn}{\lambda \exp(B^2/2)} ,\\
    m &\geq \widetilde{\Omega}\left( \lambda^{-4} n^4 \left( 1 + \left( \exp(-B^2/2) + {1/m} \right) \log^{3} (2mn/\delta) \right) \exp(-B^2) \right).
\end{align*}
From Theorem F.1 in \citep{song2021does} we know that $\lambda = \lambda_0 \exp(-B^2/2)$ for some $\lambda_0$ with no dependence on $B$ and $\lambda \exp(B^2/2) \leq 1$. 
Thus, by our constraint on $m$ and $B$, this is always satisfied. 

Finally, to require
\begin{align*}
    12\eta n c(R_w + R_b)\exp(-B^2/2) + C_2^2 \eta^2 n^2 \exp(-B^2) \leq \eta \lambda/4,
\end{align*}
we need $ \eta \leq \frac{\lambda \exp(B^2)}{5C_2^2 n^2}$. 
By our choice of $m, B$, we have
\begin{align*}
    2n^2 \exp(-\frac{2}{3} cm (R_w + R_b) \exp(-B^2/2)) = e^{-\Omega(n)}.
\end{align*}
\end{proof}

\subsection{Bounding the Number of Activated Neurons per Iteration}
First we define the set of activated neurons at iteration $t$ for training point $x_i$ to be
\begin{align*}
    \mathcal{S}_{\textnormal{on}}(i,t) = \{r \in [m]:\ w_r(t)^\top x_i \geq b_r(t)\}.
\end{align*}

\begin{lemma}[Number of Activated Neurons at Initialization]\label{lemma: number_activated_neuron_init}
Assume the choice of $m$ in \Cref{lemma: convergence}. 
With probability at least $1 - e^{-\Omega(n)}$ over the random initialization, we have
\begin{align*}
    |\mathcal{S}_{\textnormal{on}}(i,t)| = O(m \cdot \exp(-B^2/2)),
\end{align*}
for all $0 \leq t \leq T$ and $i \in [n]$. 
And As a by-product, 
\begin{align*}
    \norm{Z(0)}_F^2 \leq 8n \exp(-B^2/2).
\end{align*}
\end{lemma}
\begin{proof}
First we bound the number of activated neuron at the initialization. 
We have $\Pr[w_r^\top x_i \geq B] \leq \exp(-B^2/2)$. 
By Bernstein's inequality, 
\begin{align*}
    \Pr[|\mathcal{S}_{\textnormal{on}}(i,0)| \geq m \exp(-B^2/2) + t] \leq \exp\left( -\frac{t^2}{m \exp(-B^2/2) + t/3} \right).
\end{align*}
Take $t = m \exp(-B^2/2)$ we have
\begin{align*}
    \Pr[|\mathcal{S}_{\textnormal{on}}(i,0)| \geq 2m \exp(-B^2/2)] \leq \exp\left( -m \exp(-B^2/2) / 4 \right).
\end{align*}
By a union bound over $i \in [n]$, we have
\begin{align*}
    \Pr[\forall i \in [n]:\ |\mathcal{S}_{\textnormal{on}}(i,0)| \leq 2m \exp(-B^2/2)] \geq 1 - n \exp\left( -m \exp(-B^2/2) / 4 \right).
\end{align*}
Notice that
\begin{align*}
    \norm{Z(0)}_F^2 &\leq \frac{4}{m} \sum_{r=1}^m \sum_{i=1}^n \mathbb{I}_{r,i}(0) \leq 8n \exp(-B^2/2).
\end{align*}
\end{proof}

\begin{lemma}[Number of Activated Neurons per Iteration]
Assume the parameter settings in \Cref{lemma: convergence}. With probability at least $1 - e^{-\Omega(n)}$ over the random initialization, we have
\begin{align*}
    |\mathcal{S}_{\textnormal{on}}(i,t)| = O(m \cdot \exp(-B^2/2))
\end{align*}
for all $0 \leq t \leq T$ and $i \in [n]$. 
\end{lemma}
\begin{proof}
By \Cref{corollary: num_flipped_neurons} and \Cref{lemma: convergence}, we have
\begin{align*}
    \Pr[\forall i \in [n]:\ |\overline{S}_i| \leq 4 m c\exp(-B^2/2)] \geq 1 - e^{-\Omega(n)}.
\end{align*}
Recall $\overline{S}_i$ is the set of flipped neurons during the entire training process. 
Notice that $|\mathcal{S}_{\textnormal{on}}(i,t)| \leq |\mathcal{S}_{\textnormal{on}}(i,0)| + |\overline{S}_i|$.
Thus, by \Cref{lemma: number_activated_neuron_init}
\begin{align*}
    \Pr[\forall i \in [n]:\ |\mathcal{S}_{\textnormal{on}}(i,t)| = O(m \exp(-B^2/2))] \geq 1 - e^{-\Omega(n)}.
\end{align*}
\end{proof}

\section{Bounding the Smallest Eigenvalue with Structured Data}\label{app: res_least_eig}
\begin{theorem}\label{lemma: restricted_least_eig}
Let $X = (x_1, \ldots, x_n)$ be points in $\R^d$ with $\norm{x_i}_2 = 1$ for all $i \in [n]$ and $w \sim \mathcal{N}(0, I_d)$. Suppose that there exists $\delta \in [0,\sqrt{2}]$ such that
\begin{align*}
    \min_{i \neq j \in [n]} (\norm{x_i - x_j}_2, \norm{x_i + x_j}_2) \geq \delta.
\end{align*}
Let $B \geq 0$. Recall the limit NTK matrix $H^\infty$ defined as
\begin{align*}
    H^\infty_{ij} &:= \E_{w \sim \mathcal{N}(0, I)} \left[ (\inprod{x_i, x_j} + 1) \mathbb{I}(w^\top x_i \geq B, w^\top x_j \geq B) \right].
\end{align*}
Define $p_0 = \Pr[w^\top x_1 \geq B]$ and $p_{ij} = \Pr[w^\top x_i \geq B,\ w^\top x_j \geq B]$ for $i \neq j$.
Define the (data-dependent) region $\mathcal{R} = \{a \in \R^n:\ \sum_{i \neq j} a_i a_j p_{ij} \geq \min_{i' \neq j'} p_{i'j'} \sum_{i \neq j} a_i a_j\} $ and let $\lambda:= \min_{\norm{a}_2=1,\ a \in \mathcal{R}} a^\top H^\infty a$. 
Then, $\lambda \geq \max(0, \lambda')$ where
\begin{align*}
    \lambda' &\geq p_0 - \min_{i\neq j} p_{ij} \\
    &\geq \max\left( \frac{1}{2} - \frac{B}{\sqrt{2\pi}},\ \left(\frac{1}{B} - \frac{1}{B^3} \right)\frac{e^{-B^2/2}}{\sqrt{2\pi}} \right) - e^{-B^2/(2-\delta^2/2)} \frac{\pi - \arctan \left( \frac{\delta\sqrt{1 - \delta^2/4}}{1 - \delta^2/2} \right) }{2\pi}.
\end{align*}
\end{theorem}
\begin{proof}
Define $\Delta := \max_{i \neq j} |\inprod{x_i, x_j}|$. 
Then by our assumption, 
\begin{align*}
    & 1 - \Delta = 1 - \max_{i \neq j} |\inprod{x_i, x_j}| = \frac{\min_{i \neq j}(\norm{x_i - x_j}_2^2, \norm{x_i + x_j}_2^2)}{2} \geq \delta^2/2 \\
    &\Rightarrow \Delta \leq 1 - \delta^2/2.
\end{align*}
Further, we define
\begin{align*}
    Z(w) := [x_1 \mathbb{I}(w^\top x_1 \geq B), x_2 \mathbb{I}(w^\top x_2 \geq B), \ldots, x_n \mathbb{I}(w^\top x_n \geq B)] \in \R^{d \times n}.
\end{align*}
Notice that $H^\infty = \E_{w \sim \mathcal{N}(0, I)} \left[Z(w)^\top Z(w) +  \mathbb{I}(Xw \geq B) \mathbb{I}(Xw \geq B)^\top \right]$. 
We need to lower bound 
\begin{align*}
    \min_{\norm{a}_2=1, a \in \mathcal{R}} a^\top H^\infty a &= \min_{\norm{a}_2=1, a \in \mathcal{R}} a^\top \E_{w \sim \mathcal{N}(0, I)} \left[ Z(w)^\top Z(w) \right] a \\
    &\quad + a^\top \E_{w \sim \mathcal{N}(0, I)} \left[ \mathbb{I}(Xw \geq B) \mathbb{I}(Xw \geq B)^\top \right] a \\
    &\geq \min_{\norm{a}_2=1, a \in \mathcal{R}} a^\top \E_{w \sim \mathcal{N}(0, I)} \left[ \mathbb{I}(Xw \geq B) \mathbb{I}(Xw \geq B)^\top \right] a.
\end{align*}
Now, for a fixed $a$, 
\begin{align*}
    & a^\top \E_{w \sim \mathcal{N}(0, I)} \left[ \mathbb{I}(Xw \geq B) \mathbb{I}(Xw \geq B)^\top \right] a \\
    &=  \sum_{i=1}^n a_i^2 \Pr[w^\top x_i \geq B] + \sum_{i \neq j} a_i a_j \Pr[w^\top x_i \geq B,\ w^\top x_j \geq B] \\
    &= p_0 \norm{a}_2^2 + \sum_{i \neq j} a_i a_j p_{ij},
\end{align*}
where the last equality is by $\Pr[w^\top x_1 \geq B] = \ldots = \Pr[w^\top x_n \geq B] = p_0$ which is due to spherical symmetry of standard Gaussian. 
Notice that $\max_{i \neq j} p_{ij} \leq p_0$.
Since $a \in \mathcal{R}$, 
\begin{align*}
    \E_{w \sim \mathcal{N}(0, I)} \left[ (a^\top \mathbb{I}(Xw \geq B))^2 \right] &\geq (p_0 - \min_{i \neq j} p_{ij}) \norm{a}_2^2 + (\min_{i \neq j} p_{ij}) \norm{a}_2^2 + (\min_{i \neq j} p_{ij}) \sum_{i \neq j} a_i a_j \\
    &= (p_0 - \min_{i \neq j} p_{ij}) \norm{a}_2^2 + (\min_{i \neq j} p_{ij}) \left( \sum_i a_i \right)^2.
\end{align*}
Thus,
\begin{align*}
    \lambda &\geq \min_{\norm{a}_2=1, a \in \mathcal{R}} \E_{w \sim \mathcal{N}(0, I)} \left[ (a^\top \mathbb{I}(Xw \geq B))^2 \right] \\
    &\geq \min_{\norm{a}_2=1, a \in \mathcal{R}} (p_0 - \min_{i \neq j} p_{ij}) \norm{a}_2^2 + \min_{\norm{a}_2=1, a \in \mathcal{R}} (\min_{i \neq j} p_{ij}) \left( \sum_i a_i \right)^2 \\
    &\geq p_0 - \min_{i \neq j} p_{ij} .
\end{align*}
Now we need to upper bound
\begin{align*}
    \min_{i \neq j} p_{ij} \leq \max_{i \neq j} p_{ij}.
\end{align*}
We divide into two cases: $B = 0$ and $B > 0$. 
Consider two fixed examples $x_1, x_2$. 
Then, let $v = (I - x_1 x_1^\top) x_2 / \norm{(I - x_1 x_1^\top) x_2}$ and $c = |\inprod{x_1, x_2}|$ \footnote{Here we force $c$ to be positive. Since we are dealing with standard Gaussian, the probability is exactly the same if $c<0$ by symmetry and therefore, we force $c>0$.}.

\textbf{Case 1: $B = 0$.}
First, let us define the region $\mathcal{A}_0$ as
\begin{align*}
    \mathcal{A}_0 = \left\{ (g_1, g_2) \in \R^2:\ g_1 \geq 0,\ g_1 \geq -\frac{\sqrt{1-c^2}}{c} g_2 \right\}.
\end{align*}
Then,
\begin{align*}
    \Pr[w^\top x_1 \geq 0,\ w^\top x_2 \geq 0] &= \Pr[w^\top x_1 \geq 0,\ w^\top(c x_1 + \sqrt{1 - c^2} v) \geq 0] \\
    &= \Pr[g_1 \geq 0,\ c g_1 + \sqrt{1 - c^2} g_2 \geq 0] \\
    &= \Pr[\mathcal{A}_0] \\
    &= \frac{\pi - \arctan \left( \frac{\sqrt{1 - c^2}}{|c|} \right) }{2\pi} \\
    &\leq \frac{\pi - \arctan \left( \frac{\sqrt{1 - \Delta^2}}{|\Delta|} \right) }{2\pi},
\end{align*}
where we define $g_1:= w^\top x_1$ and $g_2 := w^\top v$ and the second equality is by the fact that since $x_1$ and $v$ are orthonormal, $g_1$ and $g_2$ are two independent standard Gaussian random variables; the last inequality is by $\arctan$ is a monotonically increasing function and $\frac{\sqrt{1-c^2}}{|c|}$ is a decreasing function in $|c|$ and $|c| \leq \Delta$. Thus,
\begin{align*}
    & \min_{i \neq j} p_{ij} \leq \max_{i \neq j} p_{ij} \leq \frac{\pi - \arctan\left( \frac{\sqrt{1 - \Delta^2}}{|\Delta|} \right)}{2\pi} .
\end{align*}

\textbf{Case 2: $B > 0$.}
First, let us define the region
\begin{align*}
    \mathcal{A} = \left\{(g_1, g_2) \in \R^2:\ g_1 \geq B,\ g_1 \geq \frac{B}{c} - \frac{\sqrt{1 - c^2}}{c} g_2 \right\}.
\end{align*}
Then, following the same steps as in case 1, we have
\begin{align*}
    \Pr[w^\top x_1 \geq B,\ w^\top x_2 \geq B] &= \Pr[g_1 \geq B,\ c g_1 + \sqrt{1 - c^2} g_2 \geq B] = \Pr[\mathcal{A}].
\end{align*}
Let $B_1 = B$ and $B_2 = B \sqrt{\frac{1-c}{1+c}}$. 
Further, notice that $\mathcal{A} = \mathcal{A}_0 + (B_1, B_2)$. Then,
\begin{align*}
    \Pr[\mathcal{A}] &= \iint_{(g_1,g_2) \in \mathcal{A}} \frac{1}{2\pi} \exp\left\{ -\frac{g_1^2 + g_2^2}{2} \right\}\ dg_1\ dg_2 \\
    &= \iint_{(g_1, g_2) \in \mathcal{A}_0} \frac{1}{2\pi} \exp\left\{ -\frac{(g_1+B_1)^2 + (g_2 + B_2)^2}{2} \right\}\ dg_1\ dg_2 \\
    &= e^{-(B_1^2 + B_2^2)/2} \iint_{(g_1, g_2) \in \mathcal{A}_0} \frac{1}{2\pi} \exp\left\{ -B_1 g_1 - B_2 g_2  \right\} \exp\left\{ -\frac{g_1^2 + g_2^2}{2} \right\}\ dg_1\ dg_2.
\end{align*}
Now, $B_1 g_1 + B_2 g_2 = B g_1 + B \sqrt{\frac{1-c}{1+c}} g_2 \geq 0$ always holds if and only if $g_1 \geq - \sqrt{\frac{1-c}{1+c}} g_2$. 
Define the region $\mathcal{A}_+$ to be
\begin{align*}
    \mathcal{A}_+ = \left\{ (g_1, g_2) \in \R^2:\ g_1 \geq 0,\ g_1 \geq - \sqrt{\frac{1-c}{1+c}} g_2 \right\}.
\end{align*}
Observe that
\begin{align*}
    \sqrt{\frac{1-c}{1+c}} \leq \frac{\sqrt{1-c^2}}{c} = \frac{\sqrt{(1-c)(1+c)}}{c} \Leftrightarrow c \leq 1 + c.
\end{align*}
Thus, $\mathcal{A}_0 \subset \mathcal{A}_+$. Therefore,
\begin{align*}
    \Pr[\mathcal{A}] &\leq e^{-(B_1^2 + B_2^2)/2} \iint_{(g_1, g_2) \in \mathcal{A}_0} \frac{1}{2\pi} \exp\left\{ -\frac{g_1^2 + g_2^2}{2} \right\}\ dg_1\ dg_2 \\
    &= e^{-(B_1^2 + B_2^2)/2} \Pr[\mathcal{A}_0] \\
    &= e^{-(B_1^2 + B_2^2)/2} \frac{\pi - \arctan \left( \frac{\sqrt{1 - c^2}}{|c|} \right) }{2\pi} \\
    &\leq e^{-B^2/(1+\Delta)} \frac{\pi - \arctan \left( \frac{\sqrt{1 - \Delta^2}}{|\Delta|} \right) }{2\pi}.
\end{align*}
Finally, we need to lower bound $p_0$. 
This can be done in two ways: when $B$ is small, we apply Gaussian anti-concentration bound and when $B$ is large, we apply Gaussian tail bounds. Thus,
\begin{align*}
    p_0 = \Pr[w^\top x_1 \geq B] \geq \max\left( \frac{1}{2} - \frac{B}{\sqrt{2\pi}},\ \left(\frac{1}{B} - \frac{1}{B^3} \right)\frac{e^{-B^2/2}}{\sqrt{2\pi}} \right).
\end{align*}
Combining the lower bound of $p_0$ and upper bound on $\max_{i\neq j} p_{ij}$ we have
\begin{align*}
    &\lambda \geq p_0 - \min_{i\neq j} p_{ij} \\
    &\quad \geq \max\left( \frac{1}{2} - \frac{B}{\sqrt{2\pi}},\ \left(\frac{1}{B} - \frac{1}{B^3} \right)\frac{e^{-B^2/2}}{\sqrt{2\pi}} \right) - e^{-B^2/(1+\Delta)} \frac{\pi - \arctan \left( \frac{\sqrt{1 - \Delta^2}}{|\Delta|} \right) }{2\pi}.
\end{align*}
Applying $\Delta \leq 1 - \delta^2/2$ and noticing that $H^\infty$ is positive semi-definite gives our final result. 
\end{proof}

\section{Generalization}\label{app: generalization}
\subsection{Rademacher Complexity}
In this section, we would like to compute the Rademacher Complexity of our network. 
Rademacher complexity is often used to bound the deviation from empirical risk and true risk (see, e.g. \citep{shalev2014understanding}.) 
\begin{definition}[Empirical Rademacher Complexity]\label{def: rademacher_complexity}
Given $n$ samples $S$, the \emph{empirical Rademacher complexity} of a function class $\mathcal{F}$, where $f: \R^d \rightarrow \R$ for $f \in \mathcal{F}$, is defined as 
\begin{align*}
    \mathcal{R}_S(\mathcal{F}) = \frac{1}{n} \E_{\eps} \left[ \sup_{f \in \mathcal{F}} \sum_{i=1}^n \eps_i f(x_i) \right]
\end{align*}
where $\eps = (\eps_1, \ldots, \eps_n)^\top$ and $\eps_i$ is an i.i.d Rademacher random variable. 
\end{definition}
\begin{theorem}[\citep{shalev2014understanding}]\label{thm: rademacher_concentration}
Suppose the loss function $\ell(\cdot, \cdot)$ is bounded in $[0,c]$ and is $\rho$-Lipschitz in the first argument. 
Then with probability at least $1 - \delta$ over sample $S$ of size $n$:
\begin{align*}
    \sup_{f \in \mathcal{F}} L_{\mathcal{D}}(f) - L_S(f) \leq 2\rho \mathcal{R}_S(\mathcal{F}) + 3c \sqrt{\frac{\log(2/\delta)}{2n}}.
\end{align*}
\end{theorem}
In order to get meaningful generalization bound via Rademacher complexity, previous results, such as \citep{arora2019fine, song2019quadratic}, multiply the neural network by a scaling factor $\kappa$ to make sure the neural network output something small at the initialization, which requires at least modifying all the previous lemmas we already established. 
We avoid repeating our arguments by utilizing symmetric initialization to force the neural network to output exactly zero for any inputs at the initialization. \footnote{While preparing the manuscript, the authors notice that this can be alternatively solved by reparameterized the neural network by $f(x; W) - f(x; W_0)$ and thus minimizing the following objective $L = \frac{1}{2} \sum_{i=1}^n (f(x_i; W) - f(x_i; W_0) - y_i)^2$. The corresponding generalization is the same since Rademacher complexity is invariant to translation. However, since the symmetric initialization is widely adopted in theory literature, we go with symmetric initialization here. }
\begin{definition}[Symmetric Initialization]
For a one-hidden layer neural network with $2m$ neurons, the network is initialized as the following 
\begin{enumerate}
    \item For $r \in [m]$, initialize $w_r \sim \mathcal{N}(0, I)$ and $a_r \sim \textnormal{Uniform}(\{-1, 1\})$. 
    \item For $r \in \{m + 1, \ldots, 2m\}$, let $w_r = w_{r-m}$ and $a_r = -a_{r - m}$. 
\end{enumerate}
\end{definition}
It is not hard to see that all of our previously established lemmas hold including expectation and concentration. 
The only effect this symmetric initialization brings is to worse the concentration by a constant factor of $2$ which can be easily addressed. 
For detailed analysis, see \citep{munteanu2022bounding}.

In order to state our final theorem, we need to use \Cref{def: non-degenerate_dist}. 
Now we can state our theorem for generalization. 
\begin{theorem}
Fix a failure probability $\delta \in (0,1)$ and an accuracy parameter $\eps \in (0, 1)$. 
Suppose the training data $S = \{(x_i,y_i)\}_{i=1}^n$ are i.i.d. samples from a $(\lambda, \delta, n)$-non-degenerate distribution $\mathcal{D}$. 
Assume the settings in \Cref{lemma: convergence} except now we let 
\begin{align*}
    m \geq \widetilde{\Omega}\left( \lambda^{-4} n^6 \left( 1 + \left( \exp(-B^2/2) + {1/m} \right) \log^{3} (2mn/\delta) \right) \exp(-B^2) \right).
\end{align*}
Consider any loss function $\ell: \R \times \R \rightarrow [0,1]$ that is $1$-Lipschitz in its first argument. 
Then with probability at least $1 - 2\delta - e^{-\Omega(n)}$ over the symmetric initialization of $W(0) \in \R^{m \times d}$ and $a \in \R^m$ and the training samples, the two layer neural network $f(W(k), b(k),a)$ trained by gradient descent for $k \geq \Omega(\frac{1}{\eta \lambda} \log \frac{n \log (1/\delta)}{\eps})$ iterations has population loss $L_{\mathcal{D}}(f) = \E_{(x,y) \sim \mathcal{D}}[\ell(f(x), y)]$ upper bounded as
\begin{align*}
    L_{\mathcal{D}}(f(W(k), b(k), a)) \leq \sqrt{\frac{y^\top (H^\infty)^{-1} y \cdot 32 \exp(-B^2/2)}{n}} + \tilde{O}\left(\frac{1}{n^{1/2}} \right).
\end{align*}
\end{theorem}
\begin{proof}
First, we need to bound $L_S$. 
After training, we have $\norm{f(k) - y}_2 \leq \eps < 1$, and thus
\begin{align*}
    L_S(f(W(k), b(k), a)) &= \frac{1}{n} \sum_{i=1}^n [\ell(f_i(k), y_i) - \ell(y_i, y_i)] \\
    &\leq \frac{1}{n} \sum_{i=1}^n |f_i(k) - y_i| \\
    &\leq \frac{1}{\sqrt{n}} \norm{f(k) - y}_2 \\
    &\leq \frac{1}{\sqrt{n}}.
\end{align*}

By \Cref{thm: rademacher_concentration}, we know that 
\begin{align*}
    L_{\mathcal{D}}(f(W(k), b(k), a)) \leq  &~ L_S(f(W(k), b(k), a)) + 2{\cal R}_S({\cal F}) + \tilde{O}(n^{-1/2})\\
    \leq &~ 2{\cal R}_S({\cal F}) + \tilde{O}(n^{-1/2}).
\end{align*}
Then, by \Cref{thm: rademacher_gd}, we get that for sufficiently large $m$, 
\begin{align*}
    \mathcal{R}_S(\mathcal{F}) \leq &~  \sqrt{\frac{y^\top (H^\infty)^{-1} y \cdot 8 \exp(-B^2/2)}{n}} + \tilde{O}\left(\frac{\exp(-B^2/4)}{n^{1/2}} \right)\\
    \leq &~ \sqrt{\frac{y^\top (H^\infty)^{-1} y \cdot 8 \exp(-B^2/2)}{n}} + \tilde{O}\left(\frac{1}{n^{1/2}} \right),
\end{align*}
where the last step follows from $B>0$.

Therefore, we conclude that:
\begin{align*}
    L_{\mathcal{D}}(f(W(k), b(k), a)) \leq \sqrt{\frac{y^\top (H^\infty)^{-1} y \cdot 32 \exp(-B^2/2)}{n}} + \tilde{O}\left(\frac{1}{n^{1/2}} \right).
\end{align*}
\end{proof}

\begin{theorem}\label{thm: rademacher_gd}
Fix a failure probability $\delta \in (0,1)$. 
Suppose the training data $S = \{(x_i,y_i)\}_{i=1}^n$ are i.i.d. samples from a $(\lambda, \delta, n)$-non-degenerate distribution $\mathcal{D}$. 
Assume the settings in \Cref{lemma: convergence} except now we let
\begin{align*}
    m &\geq \widetilde{\Omega}\left( \lambda^{-6} n^6 \left( 1 + \left( \exp(-B^2/2) + {1/m} \right) \log^{3} (2mn/\delta) \right) \exp(-B^2) \right).
\end{align*}
Denote the set of one-hidden-layer neural networks trained by gradient descent as $\mathcal{F}$. 
Then with probability at least $1 - 2\delta - e^{-\Omega(n)}$ over the randomness in the symmetric initialization and the training data, the set $\mathcal{F}$ has empirical Rademacher complexity bounded as
\begin{align*}
    \mathcal{R}_S(\mathcal{F}) &\leq \sqrt{\frac{y^\top (H^\infty)^{-1} y \cdot 8 \exp(-B^2/2)}{n}} + \tilde{O}\left(\frac{\exp(-B^2/4)}{n^{1/2}} \right).
\end{align*}
\end{theorem}
{Note that the only extra requirement we make on $m$ is the $(n/\lambda)^6$ dependence instead of $(n/\lambda)^4$ which is needed for convergence. 
The dependence of $m$ on $n$ is significantly better than previous work \citep{song2019quadratic} where the dependence is $n^{14}$. We take advantage of our initialization and new analysis to improve the dependence on $n$. }


\begin{proof}
Let $R_w$ ($R_b$) denotes the maximum distance moved any any neuron weight (bias), the same role as $D_w$ ($D_b$) in \Cref{lemma: weight_bias_movement}.
From \Cref{lemma: weight_bias_movement} and \Cref{claim: initial_error}, and we have
\begin{align*}
    \max(R_w, R_b) \leq O\left( \frac{n \sqrt{1 + (\exp(-B^2/2) + 1/m) \log^3(2mn/\delta)} }{\sqrt{m} \lambda} \right).
\end{align*}
The rest of the proof depends on the results from \Cref{lemma: rademacher_complexity_fixed_R} and \Cref{lemma: W_fro_diff}. 
Let $R := \norm{[W,b](k) - [W,b](0)}_F$.
By \Cref{lemma: rademacher_complexity_fixed_R} we have
\begin{align*}
    & \mathcal{R}_S(\mathcal{F}_{R_w, R_b, R}) \\
    &\leq R \sqrt{\frac{8 \exp(-B^2/2)}{n}} + 4c (R_w + R_b)^2 \sqrt{m} \exp(-B^2/2) \\
    &\leq R \sqrt{\frac{8 \exp(-B^2/2)}{n}} + O\left( \frac{n^2 (1 + (\exp(-B^2/2) + 1/m) \log^3(2mn/\delta)) \exp(-B^2/2)}{\sqrt{m} \lambda^2} \right).
\end{align*}
\Cref{lemma: W_fro_diff} gives that
\begin{align*}
    & R \leq \\
    & \sqrt{y^\top (H^\infty)^{-1} y} + O\left( \frac{n}{\lambda} \left( \frac{\exp(-B^2/2) \log(n/\delta)}{m} \right)^{1/4} \right) 
    + O\left( \frac{n\sqrt{(R_w+ R_b) \exp(-B^2/2)}}{\lambda} \right) \\
    &\quad + \frac{n}{\lambda^2} \cdot O\left( \exp(-B^2/4) \sqrt{\frac{\log(n^2/\delta)}{m}} + (R_w + R_b) \exp(-B^2/2) \right) .
\end{align*}
Combining the above results and using the choice of $m, R, B$ in \Cref{lemma: convergence} gives us
\begin{align*}
    & \mathcal{R}(\mathcal{F}) \\
    &\leq \sqrt{\frac{y^\top (H^\infty)^{-1} y \cdot 8 \exp(-B^2/2)}{n}} + O\left( \frac{\sqrt{n \exp(-B^2/2)}}{\lambda} \left( \frac{\exp(-B^2/2) \log(n/\delta)}{m} \right)^{1/4} \right) \\
    &\quad + O\left( \frac{\sqrt{n(R_w + R_b)}}{\lambda \exp(B^2/2)} \right) + \frac{\sqrt{n}}{\lambda^2} \cdot O\left( \exp(-B^2/2) \sqrt{\frac{\log(n^2/\delta)}{m}} + (R_w + R_b) \exp(-3B^2/4) \right) \\
    &\quad + O\left( \frac{n^2 (1 + (\exp(-B^2/2) + 1/m) \log^3(2mn/\delta)) \exp(-B^2/2)}{\sqrt{m} \lambda^2} \right). 
\end{align*}

Now, we analyze the terms one by one by plugging in the bound of $m$ and $R_w, R_b$ and show that they can be bounded by $\tilde{O}(\exp(-B^2/4)/n^{1/2})$.
For the second term, we have
\begin{align*}
    O\left( \frac{\sqrt{n \exp(-B^2/2)}}{\lambda} \left( \frac{\exp(-B^2/2) \log(n/\delta)}{m} \right)^{1/4} \right) = O\left( \frac{\sqrt{\lambda} \exp(-B^2/8) \log^{1/4}(n / \delta)}{n} \right).
\end{align*}
For the third term, we have
\begin{align*}
    O\left( \frac{\sqrt{n(R_w + R_b)}}{\lambda \exp(B^2/2)} \right) &= O\left( \frac{\sqrt{n}}{\lambda \exp(B^2/2)} \frac{\sqrt{n} (1 + (\exp(-B^2/2) + 1/m) \log^3(2mn/\delta))^{1/4}}{m^{1/4} \lambda^{1/2}} \right) \\
    &= O\left( \frac{n}{\exp(B^2/2) n^{6/4} \exp(-B^2/4)} \right) \\
    &= O \left( \frac{\exp(-B^2/4)}{n^{1/2} } \right).
\end{align*}
For the fourth term, we have 
\begin{align*}
    & \frac{\sqrt{n}}{\lambda^2} \cdot O\left( \exp(-B^2/2) \sqrt{\frac{\log(n^2/\delta)}{m}} + (R_w + R_b) \exp(-3B^2/4) \right) \\
    &= O\left( \frac{\lambda \sqrt{\log (n/\delta)}}{n^{2.5}} \right) + O\left( \frac{ \exp(-B^2/4)}{n^{1.5}} \right).
\end{align*}
For the last term, we have
\begin{align*}
    & O\left( \frac{n^2 (1 + (\exp(-B^2/2) + 1/m) \log^3(2mn/\delta)) \exp(-B^2/2)}{\sqrt{m} \lambda^2} \right) \\
    &= O \left( \frac{\lambda \sqrt{1 + (\exp(-B^2/2) + 1/m) \log^3(2mn/\delta)}}{n} \right).
\end{align*}
Recall our discussion on $\lambda$ in \Cref{sec: main_text_ideas_restricted_least_eig} that $\lambda = \lambda_0 \exp(-B^2/2) \leq 1$ for some $\lambda_0$ independent of $B$. 
Putting them together, we get the desired upper bound for ${\cal R}({\cal F})$, and the theorem is then proved.
\end{proof}

\begin{lemma}\label{lemma: rademacher_complexity_fixed_R}
Assume the choice of $R_w, R_b, m$ in \Cref{lemma: convergence}. 
Given $R > 0$, with probability at least $1 - e^{-\Omega(n)}$ over the random initialization of $W(0), a$, the following function class
\begin{align*}
    \mathcal{F}_{R_w, R_b, R} &= \{f(W, a, b):\ \norm{W - W(0)}_{2, \infty} \leq R_w,\ \norm{b - b(0)}_\infty \leq R_b,\\
    &\quad \norm{\vec{[W,b] - [W(0), b(0)]}} \leq R \}
\end{align*}
has empirical Rademacher complexity bounded as
\begin{align*}
    \mathcal{R}_S(\mathcal{F}_{R_w, R_b, R}) \leq R \sqrt{\frac{8 \exp(-B^2/2)}{n}} + 4c (R_w + R_b)^2 \sqrt{m} \exp(-B^2/2).
\end{align*}
\end{lemma}
\begin{proof}
We need to upper bound $\mathcal{R}_S(\mathcal{F}_{R_w, R_b, R})$. 
Define the events
\begin{align*}
    A_{r,i} = \{|w_r(0)^\top x_i - b_r(0)| \leq R_w + R_b\},\ i \in [n],\ r \in [m]
\end{align*}
and a shorthand $\mathbb{I}(w_r(0)^\top x_i -B \geq 0) = \mathbb{I}_{r,i}(0)$. 
Then,
\begin{align*}
    & \sum_{i=1}^n \eps_i \sum_{r=1}^m a_r \sigma(w_r^\top x_i - b_r) - \sum_{i=1}^n \eps_i \sum_{r=1}^m a_r \mathbb{I}_{r,i}(0) (w_r^\top x_i - b_r) \\
    &= \sum_{i=1}^n \sum_{r=1}^m \eps_i a_r \left( \sigma(w_r^\top x_i - b_r) - \mathbb{I}_{r,i}(0) (w_r^\top x_i - b_r) \right) \\
    &= \sum_{i=1}^n \sum_{r=1}^m \mathbb{I}(A_{r,i}) \eps_i a_r \left( \sigma(w_r^\top x_i - b_r) - \mathbb{I}_{r,i}(0) (w_r^\top x_i - b_r) \right) \\
    &= \sum_{i=1}^n \sum_{r=1}^m \mathbb{I}(A_{r,i}) \eps_i a_r \\
    &\quad \cdot \left( \sigma(w_r^\top x_i - b_r) - \mathbb{I}_{r,i}(0) (w_r(0)^\top x_i - b_r(0)) - \mathbb{I}_{r,i}(0) ((w_r - w_r(0))^\top x_i - (b_r - b_r(0))) \right) \\
    &= \sum_{i=1}^n \sum_{r=1}^m \mathbb{I}(A_{r,i}) \eps_i a_r \\
    &\quad \cdot \left( \sigma(w_r^\top x_i - b_r) - \sigma(w_r(0)^\top x_i - b_r(0)) - \mathbb{I}_{r,i}(0) ((w_r - w_r(0))^\top x_i - (b_r - b_r(0))) \right) \\
    &\leq \sum_{i=1}^n \sum_{r=1}^m \mathbb{I}(A_{r,i}) 2(R_w + R_b),
\end{align*}
where the second equality is due to the fact that $\sigma(w_r^\top x_i - b_r) = \mathbb{I}_{r,i}(0) (w_r^\top x_i - b_r)$ if $r \notin A_{r,i}$. 
Thus, the Rademacher complexity can be bounded as
\begin{align*}
    & \mathcal{R}_S(\mathcal{F}_{R_w, R_b, R}) \\
    &= \frac{1}{n} \E_{\eps} \left[ \sup_{\substack{\norm{W - W(0)}_{2,\infty} \leq R_w, \ \norm{b - b(0)}_\infty \leq R_b, \\ \norm{\vec{[W,b] - [W(0), b(0)]}} \leq R}} \sum_{i=1}^n \eps_i \sum_{r=1}^m \frac{a_r}{\sqrt{m}} \sigma(w_r^\top x_i - b_r) \right] \\
    &\leq \frac{1}{n} \E_{\eps} \left[ \sup_{\substack{\norm{W - W(0)}_{2,\infty} \leq R_w, \ \norm{b - b(0)}_\infty \leq R_b, \\ \norm{\vec{[W,b] - [W(0), b(0)]}} \leq R}} \sum_{i=1}^n \eps_i \sum_{r=1}^m \frac{a_r}{\sqrt{m}} \mathbb{I}_{r,i}(0)(w_r^\top x_i - b_r) \right] + \frac{2(R_w + R_b)}{n\sqrt{m}} \sum_{i=1}^n \sum_{r=1}^m \mathbb{I}(A_{r,i}) \\
    &= \frac{1}{n} \E_{\eps} \left[ \sup_{\norm{\vec{[W,b] - [W(0), b(0)]}} \leq R} \vec{[W,b]}^\top Z(0) \eps \right] + \frac{2(R_w + R_b)}{n\sqrt{m}} \sum_{i=1}^n \sum_{r=1}^m \mathbb{I}(A_{r,i}) \\
    &= \frac{1}{n} \E_{\eps} \left[ \sup_{\norm{\vec{[W,b] - [W(0), b(0)]}} \leq R} \vec{[W,b] - [W(0), b(0)]}^\top Z(0) \eps \right] + \frac{2(R_w + R_b)}{n\sqrt{m}} \sum_{i=1}^n \sum_{r=1}^m \mathbb{I}(A_{r,i}) \\
    &\leq \frac{1}{n} \E_\eps [R \norm{Z(0) \eps}_2 ] + \frac{2(R_w + R_b)}{n\sqrt{m}} \sum_{i=1}^n \sum_{r=1}^m \mathbb{I}(A_{r,i}) \\
    &\leq \frac{R}{n} \sqrt{\E_\eps[\norm{Z(0) \eps}_2^2]} + \frac{2(R_w + R_b)}{n\sqrt{m}} \sum_{i=1}^n \sum_{r=1}^m \mathbb{I}(A_{r,i}) \\
    &= \frac{R}{n} \norm{Z(0)}_F + \frac{2(R_w + R_b)}{n\sqrt{m}} \sum_{i=1}^n \sum_{r=1}^m \mathbb{I}(A_{r,i}),
\end{align*}
where we recall the definition of the matrix 
\begin{align*}
    Z(0) = \frac{1}{\sqrt{m}}
    \begin{bmatrix}
    \mathbb{I}_{1,1}(0) a_1 [x_1^\top,-1]^\top & \ldots & \mathbb{I}_{1,n}(0) a_1 [x_n^\top, -1]^\top \\
    \vdots & & \vdots \\
    \mathbb{I}_{m,1}(0) a_m [x_1^\top, -1]^\top & \ldots & \mathbb{I}_{m,n}(0) a_m [x_n^\top, -1]^\top 
    \end{bmatrix}
    \in \R^{m(d+1) \times n}.
\end{align*}
By \Cref{lemma: number_activated_neuron_init}, we have $\norm{Z(0)}_F \leq \sqrt{8n \exp(-B^2/2)}$ and by \Cref{corollary: num_flipped_neurons}, we have
\begin{align*}
    & \Pr\left[\forall i \in [n]:\ \sum_{r=1}^m \mathbb{I}(A_{r,i}) \leq 2 m c(R_w + R_b)\exp(-B^2/2) \right] \geq 1 - e^{-\Omega(n)}.
\end{align*}
Thus, with probability at least $1 - e^{-\Omega(n)}$, we have
\begin{align*}
    \mathcal{R}_S(\mathcal{F}_{R_w, R_b, R}) \leq R \sqrt{\frac{8 \exp(-B^2/2)}{n}} + 4c (R_w + R_b)^2 \sqrt{m} \exp(-B^2/2).
\end{align*}
\end{proof}

\subsection{Analysis of Radius}
\begin{theorem}
Assume the parameter settings in \Cref{lemma: convergence}. 
With probability at least $1 - \delta - e^{-\Omega(n)}$ over the initialization we have
\begin{align*}
    f(k) - y = - (I - \eta H^\infty)^k y \pm e(k),
\end{align*}
where 
\begin{align*}
    & \norm{e(k)}_2 \\
    &= 
    k (1 - \eta \lambda/4)^{(k-1)/2} \eta n^{3/2} \cdot O\left( \exp(-B^2/4) \sqrt{\frac{\log(n^2/\delta)}{m}} + (R_w+R_b)\exp(-B^2/2) \right).
\end{align*}
\end{theorem}
\begin{proof}
Before we start, we assume all the events needed in \Cref{lemma: convergence} succeed, which happens with probability at least $1 - \delta - e^{-\Omega(n)}$. 

Recall the no-flipping set $S_i$ in \Cref{def: noflipping_set}.
We have
\begin{align}\label{eq: f_i_update}
    & f_i(k+1) - f_i(k) \\
    &= \frac{1}{\sqrt{m}} \sum_{r=1}^m a_r [\sigma(w_r(k+1)^\top x_i - b_r(k+1)) - \sigma(w_r(k)^\top x_i - b_r(k))] \nonumber \\
    &= \frac{1}{\sqrt{m}} \sum_{r\in S_i} a_r [\sigma(w_r(k+1)^\top x_i - b_r(k+1)) - \sigma(w_r(k)^\top x_i - b_r(k))] \nonumber \\
    & \quad + \underbrace{\frac{1}{\sqrt{m}} \sum_{r \in \overline{S}_i} a_r [\sigma(w_r(k+1)^\top x_i - b_r(k+1)) - \sigma(w_r(k)^\top x_i - b_r(k))]}_{\eps_i(k)}.
\end{align}
Now, to upper bound the second term $\eps_i(k)$, 
\begin{align}\label{eq: eps_i}
    |\eps_i(k)| &= \left| \frac{1}{\sqrt{m}} \sum_{r \in \overline{S}_i} a_r [\sigma(w_r(k+1)^\top x_i - b_r(k+1)) - \sigma(w_r(k)^\top x_i - b_r(k))] \right| \nonumber \\
    &\leq \frac{1}{\sqrt{m}} \sum_{r \in \overline{S}_i} |w_r(k+1)^\top x_i - b_r(k+1) - (w_r(k)^\top x_i - b_r(k))| \nonumber \\
    &\leq \frac{1}{\sqrt{m}} \sum_{r \in \overline{S}_i} \norm{w_r(k+1) - w_r(k)}_2 + |b_r(k+1) - b_r(k)| \nonumber \\
    &= \frac{1}{\sqrt{m}} \sum_{r \in \overline{S}_i} \norm{\frac{\eta}{\sqrt{m}} a_r \sum_{j=1}^n (f_j(k) - y_j) \mathbb{I}_{r,j}(k) x_j}_2 + \left| \frac{\eta}{\sqrt{m}} a_r \sum_{j=1}^n (f_j(k) - y_j) \mathbb{I}_{r,j}(k) \right| \nonumber \\
    &\leq \frac{2\eta}{m} \sum_{r \in \overline{S}_i} \sum_{j=1}^n |f_j(k) - y_j| \nonumber \\
    &\leq \frac{2\eta \sqrt{n} |\overline{S}_i|}{m} \norm{f(k) - y}_2 \nonumber \\
    \Rightarrow \norm{\eps}_2 &= \sqrt{\sum_{i=1}^n \frac{4\eta^2 n |\overline{S}_i|^2}{m^2} \norm{f(k) - y}_2^2} \leq {\eta n O((R_w+R_b) \exp(-B^2/2))} \norm{f(k) - y}_2
\end{align}
where we apply \Cref{corollary: num_flipped_neurons} in the last inequality.
To bound the first term,
\begin{align}\label{eq: f_i_update_S_i}
    & \frac{1}{\sqrt{m}} \sum_{r\in S_i} a_r [\sigma(w_r(k+1)^\top x_i - b_r(k+1)) - \sigma(w_r(k)^\top x_i - b_r(k))] \nonumber\\
    &= \frac{1}{\sqrt{m}} \sum_{r\in S_i} a_r \mathbb{I}_{r,i}(k) \left( (w_r(k+1) - w_r(k))^\top x_i - (b_r(k+1) - b_r(k)) \right) \nonumber\\
    &= \frac{1}{\sqrt{m}} \sum_{r\in S_i} a_r \mathbb{I}_{r,i}(k) \nonumber \\
    &\quad \cdot \left( \left( -\frac{\eta}{\sqrt{m}} a_r \sum_{j=1}^n (f_j(k) - y_j) \mathbb{I}_{r,j}(k) x_j \right)^\top x_i - \frac{\eta}{\sqrt{m}} a_r \sum_{j=1}^n (f_j(k) - y_j) \mathbb{I}_{r,j}(k) \right) \nonumber\\
    &= \frac{1}{\sqrt{m}} \sum_{r\in S_i} a_r \mathbb{I}_{r,i}(k) \left( -\frac{\eta}{\sqrt{m}} a_r \sum_{j=1}^n (f_j(k) - y_j) \mathbb{I}_{r,j}(k) (x_j^\top x_i + 1) \right) \nonumber\\
    &= -\eta \sum_{j=1}^n (f_j(k) - y_j) \frac{1}{m} \sum_{r\in S_i} \mathbb{I}_{r,i}(k) \mathbb{I}_{r,j}(k) (x_j^\top x_i + 1) \nonumber\\
    &= -\eta \sum_{j=1}^n (f_j(k) - y_j) H_{ij}(k) + \underbrace{\eta \sum_{j=1}^n (f_j(k) - y_j) \frac{1}{m} \sum_{r\in \overline{S}_i} \mathbb{I}_{r,i}(k) \mathbb{I}_{r,j}(k) (x_j^\top x_i + 1)}_{\eps'_i(k)}
\end{align}
where we can upper bound $|\eps'_i(k)|$ as
\begin{align}\label{eq: eps_i'}
\begin{split}
    |\eps_i'(k)| &\leq \frac{2 \eta}{m} |\overline{S}_i| \sum_{j=1}^n |f_j(k) - y_j| \leq \frac{2\eta \sqrt{n} |\overline{S}_i|}{m} \norm{f(k) - y}_2 \\
    \Rightarrow \norm{\eps'}_2 &= \sqrt{\sum_{i=1}^n \frac{4\eta^2 n |\overline{S}_i|^2}{m^2} \norm{f(k) - y}_2^2} \leq {\eta n O((R_w+R_b) \exp(-B^2/2))} \norm{f(k) - y}_2.
\end{split}
\end{align}
Combining \Cref{eq: f_i_update}, \Cref{eq: eps_i}, \Cref{eq: f_i_update_S_i} and \Cref{eq: eps_i'}, we have
\begin{align*}
    f_{i}(k+1) - f_i(k) &= -\eta \sum_{j=1}^n (f_j(k) - y_j) H_{ij}(k) + \eps_i(k) + \eps_i'(k) \\
    \Rightarrow f(k+1) - f(k) &= -\eta H(k) (f(k) - y) + \eps(k) + \eps'(k) \\
    &= -\eta H^\infty(f(k) - y) + \underbrace{\eta (H^\infty - H(k))(f(k) - y) + \eps(k) + \eps'(k)}_{\zeta(k)} \\
    \Rightarrow f(k) - y &= (I - \eta H^\infty)^k (f(0) - y) + \sum_{t=0}^{k-1} (I - \eta H^\infty)^t \zeta(k-1-t) \\
    &= - (I - \eta H^\infty)^k y + \underbrace{(I - \eta H^\infty)^k f(0) + \sum_{t=0}^{k-1} (I - \eta H^\infty)^t \zeta(k-1-t)}_{e(k)}.
\end{align*}
Now the rest of the proof bounds the magnitude of $e(k)$. 
From \Cref{lemma: fro_diff_discrete_limit_ntk} and \Cref{lemma: perturbed_ntk}, we have
\begin{align*}
    \norm{H^\infty - H(k)}_2 &\leq \norm{H(0) - H^\infty}_2 + \norm{H(0) - H(k)}_2 \\
    &= O\left(n \exp(-B^2/4) \sqrt{\frac{\log(n^2/\delta)}{m}} \right) + O(n(R_w + R_b)\exp(-B^2/2)).
\end{align*}
Thus, we can bound $\zeta(k)$ as
\begin{align*}
    \norm{\zeta(k)}_2 &\leq \eta \norm{H^\infty - H(k)}_2 \norm{f(k) - y}_2 + \norm{\eps(k)}_2 + \norm{\eps'(k)}_2 \\
    &= O\left( \eta n \left( \exp(-B^2/4) \sqrt{\frac{\log(n^2/\delta)}{m}} + (R_w + R_b)\exp(-B^2/2) \right)  \right) \norm{f(k) - y}_2.
\end{align*}
Notice that $\norm{H^\infty}_2 \leq \tr(H^\infty) \leq n$ since $H^\infty$ is symmetric. 
By \Cref{lemma: convergence}, we pick $\eta = O(\lambda/n^2) \ll 1/\norm{H^\infty}_2$ and, with probability at least $1 - \delta - e^{-\Omega(n)}$ over the random initialization, we have $\norm{f(k) - y}_2 \leq (1 - \eta \lambda/4)^{k/2} \norm{f(0) - y}_2$. 

Since we are using symmetric initialization, we have $(I - \eta H^\infty)^k f(0) = 0$.

Thus,
\begin{align*}
    & \norm{e(k)}_2 \\
    &= \norm{\sum_{t=0}^{k-1} (I - \eta H^\infty)^t \zeta(k-1-t)}_2 \\
    &\leq \sum_{t=0}^{k-1} \norm{I - \eta H^\infty}_2^t \norm{\zeta(k-1-t)}_2 \\
    &\leq \sum_{t=0}^{k-1} (1 - \eta \lambda)^t \eta n O\left( \exp(-B^2/4) \sqrt{\frac{\log(n^2/\delta)}{m}} + (R_w + R_b)\exp(-B^2/2) \right) \\
    &\quad \cdot \norm{f(k-1-t) - y}_2 \\
    &\leq \sum_{t=0}^{k-1} (1 - \eta \lambda)^t \eta n O\left( \exp(-B^2/4) \sqrt{\frac{\log(n^2/\delta)}{m}} + (R_w + R_b)\exp(-B^2/2) \right) \\
    & \quad \cdot (1 - \eta \lambda/4)^{(k-1-t)/2} \norm{f(0) - y}_2 \\
    &\leq k (1 - \eta \lambda/4)^{(k-1)/2} \eta n O\left( \exp(-B^2/4) \sqrt{\frac{\log(n^2/\delta)}{m}} + (R_w + R_b)\exp(-B^2/2) \right) \\
    &\quad \cdot \norm{f(0) - y}_2 \\
    &\leq k (1 - \eta \lambda/4)^{(k-1)/2} \eta n^{3/2} O\Bigg( \left( \exp(-B^2/4) \sqrt{\frac{\log(n^2/\delta)}{m}} + (R_w + R_b)\exp(-B^2/2) \right) \\
    & \quad \cdot \left( \sqrt{1 + (\exp(-B^2/2) + 1/m) \log^3(2mn/\delta)} \right) \Bigg) \\
    &= k (1 - \eta \lambda/8)^{k-1} \eta n^{3/2} O\left( \exp(-B^2/4) \sqrt{\frac{\log(n^2/\delta)}{m}} + (R_w + R_b)\exp(-B^2/2) \right).
\end{align*}
\end{proof}

\begin{lemma}\label{lemma: W_fro_diff}
Assume the parameter settings in \Cref{lemma: convergence}. 
Then with probability at least $1 - \delta - e^{-\Omega(n)}$ over the random initialization, we have for all $k \geq 0$, 
\begin{align*}
    \norm{[W,b](k) - [W,b](0)}_F &\leq \sqrt{y^\top (H^\infty)^{-1} y} + O\left( \frac{n }{\lambda} \left( \frac{\exp(-B^2/2) \log(n/\delta)}{m} \right)^{1/4} \right) \\
    &\quad + O\left( \frac{n\sqrt{R \exp(-B^2/2)}}{\lambda} \right) \\
    &\quad + \frac{n}{\lambda^2} \cdot O\left( \exp(-B^2/4) \sqrt{\frac{\log(n^2/\delta)}{m}} + R\exp(-B^2/2) \right) 
\end{align*}
where $R = R_w + R_b$.
\end{lemma}
\begin{proof}
Before we start, we assume all the events needed in \Cref{lemma: convergence} succeed, which happens with probability at least $1 - \delta - e^{-\Omega(n)}$. 
\begin{align}\label{eq: W_k_W_0_diff}
    & \vec{[W,b](K)} - \vec{[W,b](0)} \nonumber \\
    &= \sum_{k=0}^{K-1} \vec{[W,b](k+1)} - \vec{[W,b](k)} \nonumber\\
    &= - \sum_{k=0}^{K-1} Z(k) (u(k) - y) \nonumber\\
    &= \sum_{k=0}^{K-1} \eta Z(k) ((I - \eta H^\infty)^k y - e(k)) \nonumber\\
    &= \sum_{k=0}^{K-1} \eta Z(k) (I - \eta H^\infty)^k y - \sum_{k=0}^{K-1} \eta Z(k) e(k) \nonumber\\
    &= \underbrace{\sum_{k=0}^{K-1} \eta Z(0) (I - \eta H^\infty)^k y}_{T_1} + \underbrace{\sum_{k=0}^{K-1} \eta (Z(k) - Z(0)) (I - \eta H^\infty)^k y}_{T_2} - \underbrace{\sum_{k=0}^{K-1} \eta Z(k) e(k)}_{T_3}.
\end{align}
Now, by \Cref{lemma: perturbed_ntk}, we have $\norm{Z(k) - Z(0)}_F \leq O(\sqrt{n R \exp(-B^2/2)})$ which implies
\begin{align}\label{eq: norm_T_2}
    \norm{T_2}_2 &= \norm{\sum_{k=0}^{K-1} \eta (Z(k) - Z(0)) (I - \eta H^\infty)^k y}_2 \nonumber\\
    &\leq \sum_{k=0}^{K-1} \eta \cdot O(\sqrt{n R \exp(-B^2/2)}) \norm{I - \eta H^\infty}_2^k \norm{y}_2 \nonumber\\
    &\leq \eta \cdot O(\sqrt{n R \exp(-B^2/2)}) \sum_{k=0}^{K-1} (1 - \eta \lambda)^k \sqrt{n} \nonumber\\
    &= O\left( \frac{n\sqrt{R \exp(-B^2/2)}}{\lambda} \right).
\end{align}
By $\norm{Z(k)}_2 \leq \norm{Z(k)}_F \leq \sqrt{2n}$, we get
\begin{align}\label{eq: norm_T_3}
    \norm{T_3}_2 &= \norm{\sum_{k=0}^{K-1} \eta Z(k) e(k)}_2 \nonumber\\
    &\leq \sum_{k=0}^{K-1} \eta \sqrt{2n} \Bigg( k (1 - \eta \lambda/8)^{k-1} \eta n^{3/2} O\left( \exp(-B^2/4) \sqrt{\frac{\log(n^2/\delta)}{m}} + R\exp(-B^2/2) \right) \Bigg) \nonumber\\
    &= \frac{n}{\lambda^2} \cdot O\left( \exp(-B^2/4) \sqrt{\frac{\log(n^2/\delta)}{m}} + R\exp(-B^2/2) \right).
\end{align}
Define $T = \eta \sum_{k=0}^{K-1} (I - \eta H^\infty)^k$. 
By \Cref{lemma: fro_diff_discrete_limit_ntk}, we know 
\begin{align*}
    \norm{H(0) - H^\infty}_2 \leq O(n \exp(-B^2/4) \sqrt{\frac{\log (n/\delta)}{m}})
\end{align*}
and this implies
\begin{align*}
    \norm{T_1}_2^2 &= \norm{\sum_{k=0}^{K-1} \eta Z(0) (I - \eta H^\infty)^k y}_2^2 \\
    &= \norm{Z(0)T y}_2^2 \\
    &= y^\top T Z(0)^\top Z(0) T y \\
    &= y^\top T H(0) T y \\
    &\leq y^\top T H^\infty T y + \norm{H(0) - H^\infty}_2 \norm{T}_2^2 \norm{y}_2^2 \\
    &\leq y^\top T H^\infty T y + O\left(n \exp(-B^2/4) \sqrt{\frac{\log (n/\delta)}{m}} \right) \left( \eta \sum_{k=0}^{K-1} (1 - \eta \lambda)^k \right)^2 n \\
    &= y^\top T H^\infty T y + O\left( \frac{n^2 \exp(-B^2/4)}{\lambda^2} \sqrt{\frac{\log (n/\delta)}{m}} \right).
\end{align*}
Let $H^\infty = U \Sigma U^\top$ be the eigendecomposition. 
Then 
\begin{align*}
    T &= U \left( \eta \sum_{k=0}^{K-1} (I - \eta \Sigma)^k \right) U^\top = U ((I - (I - \eta \Sigma)^K) \Sigma^{-1}) U^\top \\
    \Rightarrow T H^\infty T &= U ((I - (I - \eta \Sigma)^K) \Sigma^{-1})^2 \Sigma U^\top \\
    &= U (I - (I - \eta \Sigma)^K)^2 \Sigma^{-1} U^\top \\
    &\quad \preceq U \Sigma^{-1} U^\top = (H^\infty)^{-1}.
\end{align*} 
Thus,
\begin{align}\label{eq: norm_T_1}
    \norm{T_1}_2^2 &= \norm{\sum_{k=0}^{K-1} \eta Z(0) (I - \eta H^\infty)^k y}_2 \nonumber\\
    &\leq \sqrt{y^\top (H^\infty)^{-1} y + O\left( \frac{n^2 \exp(-B^2/4)}{\lambda^2} \sqrt{\frac{\log (n/\delta)}{m}} \right)} \nonumber \\
    &\leq \sqrt{y^\top (H^\infty)^{-1} y} + O\left( \frac{n }{\lambda} \left( \frac{\exp(-B^2/2) \log(n/\delta)}{m} \right)^{1/4} \right).
\end{align}
Finally, plugging in the bounds in \Cref{eq: W_k_W_0_diff}, \Cref{eq: norm_T_1}, \Cref{eq: norm_T_2}, and \Cref{eq: norm_T_3}, we have
\begin{align*}
    & \norm{[W,b](K) - [W,b](0)}_F \\
    &= \norm{\vec{[W,b](K)} - \vec{[W,b](0)}}_2 \\
    &\leq \sqrt{y^\top (H^\infty)^{-1} y} + O\left( \frac{n }{\lambda} \left( \frac{\exp(-B^2/2) \log(n/\delta)}{m} \right)^{1/4} \right) \\
    &\quad + O\left( \frac{n\sqrt{R \exp(-B^2/2)}}{\lambda} \right) \\
    &\quad + \frac{n}{\lambda^2} \cdot O\left( \exp(-B^2/4) \sqrt{\frac{\log(n^2/\delta)}{m}} + R\exp(-B^2/2) \right).
\end{align*}
\end{proof}

%% file: probability_appendix.tex
\section{Probability}
\begin{lemma}[Bernstein's Inequality]\label{lemma: bernstein}
Assume $Z_1, \ldots, Z_n$ are $n$ i.i.d. random variables with $\E[Z_i] = 0$ and $|Z_i| \leq M$ for all $i \in [n]$ almost surely. 
Let $Z = \sum_{i=1}^n Z_i$. Then, for all $t > 0$,
\begin{align*}
    \Pr[Z > t] &\leq \exp\left( -\frac{t^2 / 2}{\sum_{j=1}^n \E[Z_j^2] + Mt/3} \right)
    \leq \exp\left( - \min\left\{ \frac{t^2}{2\sum_{j=1}^n \E[Z_j^2]} , \frac{t}{2M} \right\} \right)
\end{align*}
which implies with probability at least $1 - \delta$, 
\begin{align*}
    Z \leq \sqrt{2\sum_{j=1}^n \E[Z_j^2] \log \frac{1}{\delta}} + 2M \log \frac{1}{\delta}.
\end{align*}
\end{lemma}

\begin{lemma}[Matrix Chernoff Bound, \citep{tropp2015introduction}]\label{lemma: matrix_chernoff}
Let $X_1, \ldots, X_m \in \R^{n \times n}$ be $m$ independent random Hermitian matrices. 
Assume that $0 \preceq X_i \preceq L \cdot I$ for some $L > 0$ and for all $i \in [m]$. 
Let $X:= \sum_{i=1}^m X_i$. 
Then, for $\eps \in (0, 1]$, we have
\begin{align*}
    \Pr\left[ \lambda_{\textnormal{min}}(X) \leq \eps \lambda_{\textnormal{min}}(\E[X]) \right] \leq n \cdot \exp(-(1-\eps)^2 \lambda_{\textnormal{min}}(\E[X])/(2L)).
\end{align*}
\end{lemma}

\begin{lemma}[{\citep[Theorem 3.1]{li2001gaussian}} with improved bound]\label{lemma: gaussian_range_bound}
Let $b >0$ and $r >0$. Then,
\begin{align*}
    \exp(-b^2/2) \Pr_{w \sim \mathcal{N}(0,1)} [|w| \leq r] \leq \Pr_{w \sim \mathcal{N}(0,1)}[|x - b| \leq r] \leq 2r \cdot \frac{1}{\sqrt{2\pi}} \exp(-(\max\{b-r, 0\})^2/2).
\end{align*}
\end{lemma}
\begin{proof}
To prove the upper bound, we have
\begin{align*}
    \Pr_{w \sim \mathcal{N}(0,1)}[|x - b| \leq r] = \int_{b-r}^{b+r} \frac{1}{\sqrt{2\pi}} \exp(-x^2/2)\ dx \leq 2r \cdot \frac{1}{\sqrt{2\pi}} \exp(-(\max\{b-r, 0\})^2/2).
\end{align*}
\end{proof}

\begin{lemma}[Anti-concentration of Gaussian]\label{lemma: gaussian_anticoncentration}
Let $Z \sim \mathcal{N}(0, \sigma^2)$. Then for $t>0$, 
\begin{align*}
    \Pr[|Z| \leq t] \leq \frac{2t}{\sqrt{2\pi} \sigma}.
\end{align*}
\end{lemma}




%% file: advantage_sparsity_appendix.tex
\section{The Benefit of Constant Initialization of Biases}
In short, the benefit of constant initialization of biases lies in inducing sparsity in activation and thus reducing the per step training cost. 
This is the main motivation of our work on studying sparsity from a deep learning theory perspective. 
Since our convergence shows that sparsity doesn't change convergence rate, the total training cost is also reduced. 

To address the width's dependence on $B$, our argument goes like follows. In practice, people set up neural network models by first picking a neural network of some pre-chosen size and then choose other hyper-parameters such as learning rate, initialization scale, etc. In our case, the hyper-parameter is the bias initialization. Thus, \textit{the network width is picked before $B$.} Let's say we want to apply our theoretical result to guide our practice. Since we usually don't know the exact data separation and the minimum eigenvalue of the NTK, we don't have a good estimate on the exact width needed for the network to converge and generalize. We may pick a network with width that is much larger than needed (e.g. we pick a network of width $\Omega(n^{12})$ whereas only $\Omega(n^4)$ is needed; this is possible because the smallest eigenvalue of NTK can range from $[\Omega(1/n^2), O(1)]$). Also, it is an empirical observation that the neural networks used in practice are very overparameterized and there is always room for sparsification. If the network width is very large, then per step gradient descent is very costly since the cost scales linearly with width and can be improved to scale linearly with the number of active neurons if done smartly. If the bias is initialized to zero (as people usually do in practice), then the number of active neurons is $O(m)$. However, since we can sparsify the neural network activation by non-zero bias initialization, the number of active neurons can scale \textit{sub-linearly} in $m$. Thus, if the neural network width we choose at the beginning is much larger than needed, then we are indeed able to obtain total training cost reduction by this initialization. The above is an informal description of the result proven in \citep{song2021does} and the message is \textit{sparsity can help reduce the per step training cost}. If the network width is pre-chosen, then the lower bound on network width $m \geq \Tilde{\Omega}(\lambda_0^{-4} n^4 \exp(B^2))$ in Theorem 3.1 can be translated into an upper bound on bias initialization: $B \leq \tilde{O}(\sqrt{\log \frac{\lambda_0^4 m}{n^4}})$ if $m \geq \tilde{\Omega}(\lambda_0^{-4} n^4)$. This would be a more appropriate interpretation of our result. 
Note that this is different from how Theorem 3.1 is presented: first pick $B$ and then choose $m$; since $m$ is picked later, $m$ can always satisfy $B \leq \sqrt{0.5 \log m}$ and $m \geq \Tilde{\Omega}(\lambda_0^{-4} n^4 \exp(B^2))$. 
Of course, we don't know the best (largest) possible $B$ that works but as long as we can get some $B$ to work, we can get computational gain from sparsity. 

In summary, sparsity can reduce the per step training cost since we don't know the exact width needed for the network to converge and generalize. Our result should be interpreted as an upper bound on $B$ since the width is always chosen before $B$ in practice.